\documentclass{article}
\usepackage{ijcai23}

\usepackage{times}
\usepackage{soul}
\usepackage{url}
\usepackage[hidelinks]{hyperref}
\usepackage[utf8]{inputenc}
\usepackage[small]{caption}
\usepackage{graphicx}
\usepackage{amsmath}
\usepackage{amsthm}
\usepackage{booktabs}
\usepackage{algorithm}
\usepackage[switch]{lineno}

\usepackage{bbm}
\usepackage{graphicx}
\usepackage[dvipsnames]{xcolor}
\usepackage{xspace}
\usepackage[noend]{algpseudocode}
\usepackage{subfigure}
\usepackage{booktabs}
\usepackage{placeins}
\usepackage{amsmath}
\usepackage{enumitem,amssymb}
\usepackage{wrapfig}

\usepackage{amsmath,amsfonts,bm}
\usepackage{amssymb}
\usepackage{amsthm}

\usepackage[
textsize=scriptsize,
]{todonotes}

\usepackage{lineno} 

\DeclareMathOperator*{\argmax}{arg\,max}
\DeclareMathOperator*{\argmin}{arg\,min}

\usepackage{etoolbox}

\newif\iffinal

\iffinal
    \newcommand{\fix}[1]{}
    \newcommand{\yuxin}[1]{}
    \newcommand{\yeol}[1]{}
    \newcommand{\yeolnote}[1]{}
    \newcommand{\yuxininline}[1]{}
\else
    \newcommand{\yeol}[1]{{\color{black} #1}}
    \newcommand{\yeolnote}[1]{{\textcolor{blue}{[{\bf Yeol:} #1]}}}
    \newcommand{\fix}[1]{{\color{red} #1}}
    \newcommand{\yuxin}[1]{\todo[fancyline,color=NavyBlue!40]{YC: #1}\xspace}
    \newcommand{\yuxininline}[1]{\textcolor{NavyBlue}{[YC: #1]}}
\fi

\newcommand{\denselist}{\itemsep 0pt\topsep-10pt\partopsep-3pt}

\newtheorem{theorem}{Theorem}

\title{Efficient Online Decision Tree Learning with Active Feature Acquisition}

\author{
Arman Rahbar$^1$
\and
Ziyu Ye$^2$\and
Yuxin Chen$^2$\And
Morteza Haghir Chehreghani$^1$
\affiliations
$^1$Chalmers University of Technology\\
$^2$University of Chicago
\emails
armanr@chalmers.se,
ziyuye@uchicago.edu,
chenyuxin@uchicago.edu,
morteza.chehreghani@chalmers.se
}

\begin{document}


\newcommand{\figleft}{{\em (Left)}}
\newcommand{\figcenter}{{\em (Center)}}
\newcommand{\figright}{{\em (Right)}}
\newcommand{\figtop}{{\em (Top)}}
\newcommand{\figbottom}{{\em (Bottom)}}
\newcommand{\captiona}{{\em (a)}}
\newcommand{\captionb}{{\em (b)}}
\newcommand{\captionc}{{\em (c)}}
\newcommand{\captiond}{{\em (d)}}

\newcommand{\newterm}[1]{{\bf #1}}

\def\figref#1{figure~\ref{#1}}
\def\Figref#1{Figure~\ref{#1}}
\def\twofigref#1#2{figures \ref{#1} and \ref{#2}}
\def\quadfigref#1#2#3#4{figures \ref{#1}, \ref{#2}, \ref{#3} and \ref{#4}}
\def\secref#1{section~\ref{#1}}
\def\Secref#1{Section~\ref{#1}}
\def\twosecrefs#1#2{sections \ref{#1} and \ref{#2}}
\def\secrefs#1#2#3{sections \ref{#1}, \ref{#2} and \ref{#3}}
\def\eqref#1{equation~\ref{#1}}
\def\Eqref#1{Equation~\ref{#1}}
\def\plaineqref#1{\ref{#1}}
\def\chapref#1{chapter~\ref{#1}}
\def\Chapref#1{Chapter~\ref{#1}}
\def\rangechapref#1#2{chapters\ref{#1}--\ref{#2}}
\def\algref#1{algorithm~\ref{#1}}
\def\Algref#1{Algorithm~\ref{#1}}
\def\twoalgref#1#2{algorithms \ref{#1} and \ref{#2}}
\def\Twoalgref#1#2{Algorithms \ref{#1} and \ref{#2}}
\def\partref#1{part~\ref{#1}}
\def\Partref#1{Part~\ref{#1}}
\def\twopartref#1#2{parts \ref{#1} and \ref{#2}}

\def\ceil#1{\lceil #1 \rceil}
\def\floor#1{\lfloor #1 \rfloor}
\def\1{\bm{1}}
\newcommand{\train}{\mathcal{D}}
\newcommand{\valid}{\mathcal{D_{\mathrm{valid}}}}
\newcommand{\test}{\mathcal{D_{\mathrm{test}}}}

\def\eps{{\epsilon}}

\newcommand{\defref}[1]{Definition~\ref{#1}}
\newcommand{\tableref}[1]{Table~\ref{#1}}
\newcommand{\eqnref}[1]{\text{Eq.}~(\ref{#1})}
\newcommand{\appref}[1]{Appendix \ref{#1}}
\newcommand{\thmref}[1]{Theorem~\ref{#1}}
\newcommand{\corref}[1]{Corollary~\ref{#1}}
\newcommand{\propref}[1]{Proposition~\ref{#1}}
\newcommand{\lemref}[1]{Lemma~\ref{#1}}
\renewcommand{\algref}[1]{Algorithm~\ref{#1}}
\newcommand{\egref}[1]{Example~\ref{#1}}
\newcommand{\lnref}[1]{Line~\ref{#1}}

\def\reta{{\textnormal{$\eta$}}}
\def\ra{{\textnormal{a}}}
\def\rb{{\textnormal{b}}}
\def\rc{{\textnormal{c}}}
\def\rd{{\textnormal{d}}}
\def\re{{\textnormal{e}}}
\def\rf{{\textnormal{f}}}
\def\rg{{\textnormal{g}}}
\def\rh{{\textnormal{h}}}
\def\ri{{\textnormal{i}}}
\def\rj{{\textnormal{j}}}
\def\rk{{\textnormal{k}}}
\def\rl{{\textnormal{l}}}
\def\rn{{\textnormal{n}}}
\def\ro{{\textnormal{o}}}
\def\rp{{\textnormal{p}}}
\def\rq{{\textnormal{q}}}
\def\rr{{\textnormal{r}}}
\def\rs{{\textnormal{s}}}
\def\rt{{\textnormal{t}}}
\def\ru{{\textnormal{u}}}
\def\rv{{\textnormal{v}}}
\def\rw{{\textnormal{w}}}
\def\rx{{\textnormal{x}}}
\def\ry{{\textnormal{y}}}
\def\rz{{\textnormal{z}}}

\def\rvepsilon{{\mathbf{\epsilon}}}
\def\rvtheta{{\mathbf{\theta}}}
\def\boldtheta{{\boldsymbol{\theta}}}
\def\boldalpha{{\boldsymbol{\alpha}}}
\def\boldbeta{{\boldsymbol{\beta}}}
\def\rva{{\mathbf{a}}}
\def\rvb{{\mathbf{b}}}
\def\rvc{{\mathbf{c}}}
\def\rvd{{\mathbf{d}}}
\def\rve{{\mathbf{e}}}
\def\rvf{{\mathbf{f}}}
\def\rvg{{\mathbf{g}}}
\def\rvh{{\mathbf{h}}}
\def\rvu{{\mathbf{i}}}
\def\rvj{{\mathbf{j}}}
\def\rvk{{\mathbf{k}}}
\def\rvl{{\mathbf{l}}}
\def\rvm{{\mathbf{m}}}
\def\rvn{{\mathbf{n}}}
\def\rvo{{\mathbf{o}}}
\def\rvp{{\mathbf{p}}}
\def\rvq{{\mathbf{q}}}
\def\rvr{{\mathbf{r}}}
\def\rvs{{\mathbf{s}}}
\def\rvt{{\mathbf{t}}}
\def\rvu{{\mathbf{u}}}
\def\rvv{{\mathbf{v}}}
\def\rvw{{\mathbf{w}}}
\def\rvx{{\mathbf{x}}}
\def\rvy{{\mathbf{y}}}
\def\rvz{{\mathbf{z}}}

\def\erva{{\textnormal{a}}}
\def\ervb{{\textnormal{b}}}
\def\ervc{{\textnormal{c}}}
\def\ervd{{\textnormal{d}}}
\def\erve{{\textnormal{e}}}
\def\ervf{{\textnormal{f}}}
\def\ervg{{\textnormal{g}}}
\def\ervh{{\textnormal{h}}}
\def\ervi{{\textnormal{i}}}
\def\ervj{{\textnormal{j}}}
\def\ervk{{\textnormal{k}}}
\def\ervl{{\textnormal{l}}}
\def\ervm{{\textnormal{m}}}
\def\ervn{{\textnormal{n}}}
\def\ervo{{\textnormal{o}}}
\def\ervp{{\textnormal{p}}}
\def\ervq{{\textnormal{q}}}
\def\ervr{{\textnormal{r}}}
\def\ervs{{\textnormal{s}}}
\def\ervt{{\textnormal{t}}}
\def\ervu{{\textnormal{u}}}
\def\ervv{{\textnormal{v}}}
\def\ervw{{\textnormal{w}}}
\def\ervx{{\textnormal{x}}}
\def\ervy{{\textnormal{y}}}
\def\ervz{{\textnormal{z}}}

\def\rmA{{\mathbf{A}}}
\def\rmB{{\mathbf{B}}}
\def\rmC{{\mathbf{C}}}
\def\rmD{{\mathbf{D}}}
\def\rmE{{\mathbf{E}}}
\def\rmF{{\mathbf{F}}}
\def\rmG{{\mathbf{G}}}
\def\rmH{{\mathbf{H}}}
\def\rmI{{\mathbf{I}}}
\def\rmJ{{\mathbf{J}}}
\def\rmK{{\mathbf{K}}}
\def\rmL{{\mathbf{L}}}
\def\rmM{{\mathbf{M}}}
\def\rmN{{\mathbf{N}}}
\def\rmO{{\mathbf{O}}}
\def\rmP{{\mathbf{P}}}
\def\rmQ{{\mathbf{Q}}}
\def\rmR{{\mathbf{R}}}
\def\rmS{{\mathbf{S}}}
\def\rmT{{\mathbf{T}}}
\def\rmU{{\mathbf{U}}}
\def\rmV{{\mathbf{V}}}
\def\rmW{{\mathbf{W}}}
\def\rmX{{\mathbf{X}}}
\def\rmY{{\mathbf{Y}}}
\def\rmZ{{\mathbf{Z}}}

\def\ermA{{\textnormal{A}}}
\def\ermB{{\textnormal{B}}}
\def\ermC{{\textnormal{C}}}
\def\ermD{{\textnormal{D}}}
\def\ermE{{\textnormal{E}}}
\def\ermF{{\textnormal{F}}}
\def\ermG{{\textnormal{G}}}
\def\ermH{{\textnormal{H}}}
\def\ermI{{\textnormal{I}}}
\def\ermJ{{\textnormal{J}}}
\def\ermK{{\textnormal{K}}}
\def\ermL{{\textnormal{L}}}
\def\ermM{{\textnormal{M}}}
\def\ermN{{\textnormal{N}}}
\def\ermO{{\textnormal{O}}}
\def\ermP{{\textnormal{P}}}
\def\ermQ{{\textnormal{Q}}}
\def\ermR{{\textnormal{R}}}
\def\ermS{{\textnormal{S}}}
\def\ermT{{\textnormal{T}}}
\def\ermU{{\textnormal{U}}}
\def\ermV{{\textnormal{V}}}
\def\ermW{{\textnormal{W}}}
\def\ermX{{\textnormal{X}}}
\def\ermY{{\textnormal{Y}}}
\def\ermZ{{\textnormal{Z}}}

\def\vzero{{\bm{0}}}
\def\vone{{\bm{1}}}
\def\vmu{{\bm{\mu}}}
\def\vtheta{{\bm{\theta}}}
\def\va{{\bm{a}}}
\def\vb{{\bm{b}}}
\def\vc{{\bm{c}}}
\def\vd{{\bm{d}}}
\def\ve{{\bm{e}}}
\def\vf{{\bm{f}}}
\def\vg{{\bm{g}}}
\def\vh{{\bm{h}}}
\def\vi{{\bm{i}}}
\def\vj{{\bm{j}}}
\def\vk{{\bm{k}}}
\def\vl{{\bm{l}}}
\def\vm{{\bm{m}}}
\def\vn{{\bm{n}}}
\def\vo{{\bm{o}}}
\def\vp{{\bm{p}}}
\def\vq{{\bm{q}}}
\def\vr{{\bm{r}}}
\def\vs{{\bm{s}}}
\def\vt{{\bm{t}}}
\def\vu{{\bm{u}}}
\def\vv{{\bm{v}}}
\def\vw{{\bm{w}}}
\def\vx{{\bm{x}}}
\def\vy{{\bm{y}}}
\def\vz{{\bm{z}}}

\def\evalpha{{\alpha}}
\def\evbeta{{\beta}}
\def\evepsilon{{\epsilon}}
\def\evlambda{{\lambda}}
\def\evomega{{\omega}}
\def\evmu{{\mu}}
\def\evpsi{{\psi}}
\def\evsigma{{\sigma}}
\def\evtheta{{\theta}}
\def\eva{{a}}
\def\evb{{b}}
\def\evc{{c}}
\def\evd{{d}}
\def\eve{{e}}
\def\evf{{f}}
\def\evg{{g}}
\def\evh{{h}}
\def\evi{{i}}
\def\evj{{j}}
\def\evk{{k}}
\def\evl{{l}}
\def\evm{{m}}
\def\evn{{n}}
\def\evo{{o}}
\def\evp{{p}}
\def\evq{{q}}
\def\evr{{r}}
\def\evs{{s}}
\def\evt{{t}}
\def\evu{{u}}
\def\evv{{v}}
\def\evw{{w}}
\def\evx{{x}}
\def\evy{{y}}
\def\evz{{z}}

\def\mA{{\bm{A}}}
\def\mB{{\bm{B}}}
\def\mC{{\bm{C}}}
\def\mD{{\bm{D}}}
\def\mE{{\bm{E}}}
\def\mF{{\bm{F}}}
\def\mG{{\bm{G}}}
\def\mH{{\bm{H}}}
\def\mI{{\bm{I}}}
\def\mJ{{\bm{J}}}
\def\mK{{\bm{K}}}
\def\mL{{\bm{L}}}
\def\mM{{\bm{M}}}
\def\mN{{\bm{N}}}
\def\mO{{\bm{O}}}
\def\mP{{\bm{P}}}
\def\mQ{{\bm{Q}}}
\def\mR{{\bm{R}}}
\def\mS{{\bm{S}}}
\def\mT{{\bm{T}}}
\def\mU{{\bm{U}}}
\def\mV{{\bm{V}}}
\def\mW{{\bm{W}}}
\def\mX{{\bm{X}}}
\def\mY{{\bm{Y}}}
\def\mZ{{\bm{Z}}}
\def\mBeta{{\bm{\beta}}}
\def\mPhi{{\bm{\Phi}}}
\def\mLambda{{\bm{\Lambda}}}
\def\mSigma{{\bm{\Sigma}}}

\newcommand{\tens}[1]{\bm{\mathsfit{#1}}}
\def\tA{{\tens{A}}}
\def\tB{{\tens{B}}}
\def\tC{{\tens{C}}}
\def\tD{{\tens{D}}}
\def\tE{{\tens{E}}}
\def\tF{{\tens{F}}}
\def\tG{{\tens{G}}}
\def\tH{{\tens{H}}}
\def\tI{{\tens{I}}}
\def\tJ{{\tens{J}}}
\def\tK{{\tens{K}}}
\def\tL{{\tens{L}}}
\def\tM{{\tens{M}}}
\def\tN{{\tens{N}}}
\def\tO{{\tens{O}}}
\def\tP{{\tens{P}}}
\def\tQ{{\tens{Q}}}
\def\tR{{\tens{R}}}
\def\tS{{\tens{S}}}
\def\tT{{\tens{T}}}
\def\tU{{\tens{U}}}
\def\tV{{\tens{V}}}
\def\tW{{\tens{W}}}
\def\tX{{\tens{X}}}
\def\tY{{\tens{Y}}}
\def\tZ{{\tens{Z}}}

\def\gA{{\mathcal{A}}}
\def\gB{{\mathcal{B}}}
\def\gC{{\mathcal{C}}}
\def\gD{{\mathcal{D}}}
\def\gE{{\mathcal{E}}}
\def\gF{{\mathcal{F}}}
\def\gG{{\mathcal{G}}}
\def\gH{{\mathcal{H}}}
\def\gI{{\mathcal{I}}}
\def\gJ{{\mathcal{J}}}
\def\gK{{\mathcal{K}}}
\def\gL{{\mathcal{L}}}
\def\gM{{\mathcal{M}}}
\def\gN{{\mathcal{N}}}
\def\gO{{\mathcal{O}}}
\def\gP{{\mathcal{P}}}
\def\gQ{{\mathcal{Q}}}
\def\gR{{\mathcal{R}}}
\def\gS{{\mathcal{S}}}
\def\gT{{\mathcal{T}}}
\def\gU{{\mathcal{U}}}
\def\gV{{\mathcal{V}}}
\def\gW{{\mathcal{W}}}
\def\gX{{\mathcal{X}}}
\def\gY{{\mathcal{Y}}}
\def\gZ{{\mathcal{Z}}}

\def\sA{{\mathbb{A}}}
\def\sB{{\mathbb{B}}}
\def\sC{{\mathbb{C}}}
\def\sD{{\mathbb{D}}}
\def\sF{{\mathbb{F}}}
\def\sG{{\mathbb{G}}}
\def\sH{{\mathbb{H}}}
\def\sI{{\mathbb{I}}}
\def\sJ{{\mathbb{J}}}
\def\sK{{\mathbb{K}}}
\def\sL{{\mathbb{L}}}
\def\sM{{\mathbb{M}}}
\def\sN{{\mathbb{N}}}
\def\sO{{\mathbb{O}}}
\def\sP{{\mathbb{P}}}
\def\sQ{{\mathbb{Q}}}
\def\sR{{\mathbb{R}}}
\def\sS{{\mathbb{S}}}
\def\sT{{\mathbb{T}}}
\def\sU{{\mathbb{U}}}
\def\sV{{\mathbb{V}}}
\def\sW{{\mathbb{W}}}
\def\sX{{\mathbb{X}}}
\def\sY{{\mathbb{Y}}}
\def\sZ{{\mathbb{Z}}}

\def\calA{{\mathcal{A}}}
\def\calB{{\mathcal{B}}}
\def\calC{{\mathcal{C}}}
\def\calD{{\mathcal{D}}}
\def\calE{{\mathcal{F}}}
\def\calF{{\mathcal{F}}}
\def\calG{{\mathcal{G}}}
\def\calH{{\mathcal{H}}}
\def\calI{{\mathcal{I}}}
\def\calJ{{\mathcal{J}}}
\def\calK{{\mathcal{K}}}
\def\calL{{\mathcal{L}}}
\def\calM{{\mathcal{M}}}
\def\calN{{\mathcal{N}}}
\def\calO{{\mathcal{O}}}
\def\calP{{\mathcal{P}}}
\def\calQ{{\mathcal{Q}}}
\def\calR{{\mathcal{R}}}
\def\calS{{\mathcal{S}}}
\def\calT{{\mathcal{T}}}
\def\calU{{\mathcal{U}}}
\def\calV{{\mathcal{V}}}
\def\calW{{\mathcal{W}}}
\def\calX{{\mathcal{X}}}
\def\calY{{\mathcal{Y}}}
\def\calZ{{\mathcal{Z}}}

\def\emLambda{{\Lambda}}
\def\emA{{A}}
\def\emB{{B}}
\def\emC{{C}}
\def\emD{{D}}
\def\emE{{E}}
\def\emF{{F}}
\def\emG{{G}}
\def\emH{{H}}
\def\emI{{I}}
\def\emJ{{J}}
\def\emK{{K}}
\def\emL{{L}}
\def\emM{{M}}
\def\emN{{N}}
\def\emO{{O}}
\def\emP{{P}}
\def\emQ{{Q}}
\def\emR{{R}}
\def\emS{{S}}
\def\emT{{T}}
\def\emU{{U}}
\def\emV{{V}}
\def\emW{{W}}
\def\emX{{X}}
\def\emY{{Y}}
\def\emZ{{Z}}
\def\emSigma{{\Sigma}}

\newcommand{\etens}[1]{\mathsfit{#1}}
\def\etLambda{{\etens{\Lambda}}}
\def\etA{{\etens{A}}}
\def\etB{{\etens{B}}}
\def\etC{{\etens{C}}}
\def\etD{{\etens{D}}}
\def\etE{{\etens{E}}}
\def\etF{{\etens{F}}}
\def\etG{{\etens{G}}}
\def\etH{{\etens{H}}}
\def\etI{{\etens{I}}}
\def\etJ{{\etens{J}}}
\def\etK{{\etens{K}}}
\def\etL{{\etens{L}}}
\def\etM{{\etens{M}}}
\def\etN{{\etens{N}}}
\def\etO{{\etens{O}}}
\def\etP{{\etens{P}}}
\def\etQ{{\etens{Q}}}
\def\etR{{\etens{R}}}
\def\etS{{\etens{S}}}
\def\etT{{\etens{T}}}
\def\etU{{\etens{U}}}
\def\etV{{\etens{V}}}
\def\etW{{\etens{W}}}
\def\etX{{\etens{X}}}
\def\etY{{\etens{Y}}}
\def\etZ{{\etens{Z}}}

\newcommand{\pdata}{p_{\rm{data}}}
\newcommand{\ptrain}{\hat{p}_{\rm{data}}}
\newcommand{\Ptrain}{\hat{P}_{\rm{data}}}
\newcommand{\pmodel}{p_{\rm{model}}}
\newcommand{\Pmodel}{P_{\rm{model}}}
\newcommand{\ptildemodel}{\tilde{p}_{\rm{model}}}
\newcommand{\pencode}{p_{\rm{encoder}}}
\newcommand{\pdecode}{p_{\rm{decoder}}}
\newcommand{\precons}{p_{\rm{reconstruct}}}

\newcommand{\laplace}{\mathrm{Laplace}} 

\newcommand{\E}{\mathbb{E}}
\newcommand{\Ls}{\mathcal{L}}
\newcommand{\R}{\mathbb{R}}
\newcommand{\emp}{\tilde{p}}
\newcommand{\lr}{\alpha}
\newcommand{\reg}{\lambda}
\newcommand{\rect}{\mathrm{rectifier}}
\newcommand{\softmax}{\mathrm{softmax}}
\newcommand{\sigmoid}{\sigma}
\newcommand{\softplus}{\zeta}
\newcommand{\KL}{D_{\mathrm{KL}}}
\newcommand{\Var}{\mathrm{Var}}
\newcommand{\standarderror}{\mathrm{SE}}
\newcommand{\Cov}{\mathrm{Cov}}
\newcommand{\normlzero}{L^0}
\newcommand{\normlone}{L^1}
\newcommand{\normltwo}{L^2}
\newcommand{\normlp}{L^p}
\newcommand{\normmax}{L^\infty}

\newcommand{\parents}{Pa} 

\let\ab\allowbreak


\newtheorem{corollary}[theorem]{Corollary}
\newtheorem{problem}[theorem]{Problem}

\newtheorem{lemma}[theorem]{Lemma}
\newtheorem{proposition}[theorem]{Proposition}
\newtheorem{definition}[theorem]{Definition}
\newtheorem{rem}[theorem]{Remark}
\numberwithin{equation}{section}


\newcommand{\paren} [1] {\ensuremath{ \left( {#1} \right) }}
\newcommand{\bigparen} [1] {\ensuremath{ \Big( {#1} \Big) }}
\newcommand{\bracket}[1]{\left[#1\right]}

\renewcommand{\Pr}[1]{\ensuremath{\mathbb{P}\left[#1\right] }}
\newcommand{\expct}[1]{\mathbb{E}\left[#1\right]}
\newcommand{\expctover}[2]{\mathbb{E}_{#1}\!\left[#2\right]}
\newcommand{\abs}[1]{\left\vert#1\right\vert}
\def \argmax {\mathop{\rm arg\,max}}
\def \argmin {\mathop{\rm arg\,min}}

\newcommand{\littleO}[1]{\ensuremath{o\paren{#1}}}
\newcommand{\bigO}[1]{\ensuremath{\mathcal{O}\paren{#1}}}
\newcommand{\bigTheta}[1]{\ensuremath{\Theta\paren{#1}}}
\newcommand{\bigOmega}[1]{\ensuremath{\Omega\paren{#1}}}

\newcommand{\bigOTilde}[1]{\ensuremath{\tilde{O}\paren{#1}}}

\newcommand{\NonNegativeReals}{\ensuremath{\mathbb{R}_{\ge 0}}}
\newcommand{\PositiveIntegers}{\ensuremath{\mathbb{Z^+}}}
\newcommand{\integers}{\ensuremath{\mathbb{Z}}}
\newcommand{\nats}{\ensuremath{\mathbb{N}}}
\newcommand{\reals}{\ensuremath{\mathbb{R}}}
\newcommand{\rationals}{\ensuremath{\mathbb{Q}}}
\newcommand{\distrib}[0]{\ensuremath{\mathcal{D}}}
\newcommand{\matroid}{\ensuremath{\mathcal{M}}}

\newboolean{showcomments}
\setboolean{showcomments}{true}

\newcommand{\abTar}{L^1}
\newcommand{\normalTarCdf}{F^0}
\newcommand{\normalTarQtl}{Q^0}
\newcommand{\normalTarCdfEmp}{F^0_n}
\newcommand{\abnTarCdf}{F^a}
\newcommand{\abnTarCdfEmp}{F^a_n}

\newcommand{\cuparrow}{\color{ForestGreen}{{\boldsymbol{\uparrow}}}}
\newcommand{\cdownarrow}{\color{Maroon}{{\boldsymbol{\downarrow}}}}
\newcommand{\keep}{\color{White}{{\boldsymbol{\downarrow}}}}

\maketitle

\begin{abstract}
Constructing decision trees online is a classical machine learning problem. Existing works often assume that features are readily available for each incoming data point. However, in many real world applications, both feature values and the labels are unknown \textit{a priori} and can only be obtained at a cost. 
For example, in medical diagnosis, doctors have to choose which tests to perform (\textit{i.e.}, making costly feature queries) on a patient in order to make a diagnosis decision (\textit{i.e.}, predicting labels). We provide a fresh perspective to tackle this practical challenge. 
Our framework consists of an active planning oracle embedded in an online learning scheme for which we investigate several information acquisition functions.
Specifically, we employ a surrogate information acquisition function based on adaptive submodularity
to actively query feature values with a minimal cost, while using a posterior sampling scheme to maintain a low regret for online prediction. We demonstrate the efficiency and effectiveness of our framework via extensive experiments on various real-world datasets. Our framework also naturally adapts to the challenging setting of online learning with concept drift
and is shown to be competitive  with baseline models while being more flexible. 
\end{abstract}

\section{Introduction}\label{sec:intro}

\yeol{
Decision trees constitute one of the most fundamental and crucial machine learning models, due to their interpretability and extensibility. An important variant developed for online setting has been employed in various impactful real-world applications such as medical diagnosis~\cite{podgorelec2002decision}, intrution detection~\cite{jiang2013incremental}, network troubleshooting~\cite{rozaki2015design}, etc.


Classical models aim to construct online decision trees incrementally with streaming data. However, such models have several disadvantages. First, they require that all features are presented to determine splitting node (\cite{das2019learn,feraud2016random}). However, querying feature values can be costly in real-world scenarios, e.g., conducting medical tests for medical diagnosis can be quite expensive. Second, classical models are typically not \textit{fully} trained online, where labels are assumed to be known for each point in the data steam \cite{shim2018joint}. In contrast, our work takes feature acquisition cost (formally defined in \secref{sec:formulation})  
into considerations and aim at the more challenging fully online case: we receive a data point at each step, and need to make accurate prediction of its label with low feature acquisition cost; the true label will only be observed \textit{after} we make the prediction.

Concretely, consider the medical diagnosis problem: at each round $t$, a patient $\vx ^t$ comes in, and the system is asked to predict the treatment $\vy ^t$ for the patient. Naturally, we take results of medical tests (e.g., a CT scan) as patients' features: assume that there exist $n$ medical tests, each patient $\vx ^t$ can be represented by $(x_1^t, x_2^t, \ldots, x_n^t)$, where $x_{i}^t$ is their $i$-th test results. To reiterate, the problem has two crucial charateristics, making the setting challenging yet more practical: (1) \textbf{cost of feature query}: we assume that features of a data point is initially unknown but can be acquired with a cost; (2) \textbf{(fully) online learning}: we assume that we only have prior belief on the data and have to refine it via online interaction with the data streams (i.e., patients) whose labels (i.e., treatments) are initially unknown in each round.

Our goal is to construct online decision trees efficiently. Specifically, we interpret the efficiency of our framework from two aspects: first, it requires less streaming data points (or time steps) to learn a well-performed decision tree, i.e., \textit{faster learning}; second, it incurs lower cost for the label prediction for each data point, i.e., \textit{cheaper prediction}. We refer our framework as \textbf{UFODT}, i.e., \textbf{U}tility of \textbf{F}eatures for \textbf{O}nline learning of \textbf{D}ecision \textbf{T}rees. As shown in Figure~\ref{fig:framework}, our framework can be interpreted as an active planning oracle nested within an online learning model.

\begin{figure}[t]
\centering
 \includegraphics[width=0.35\textwidth]{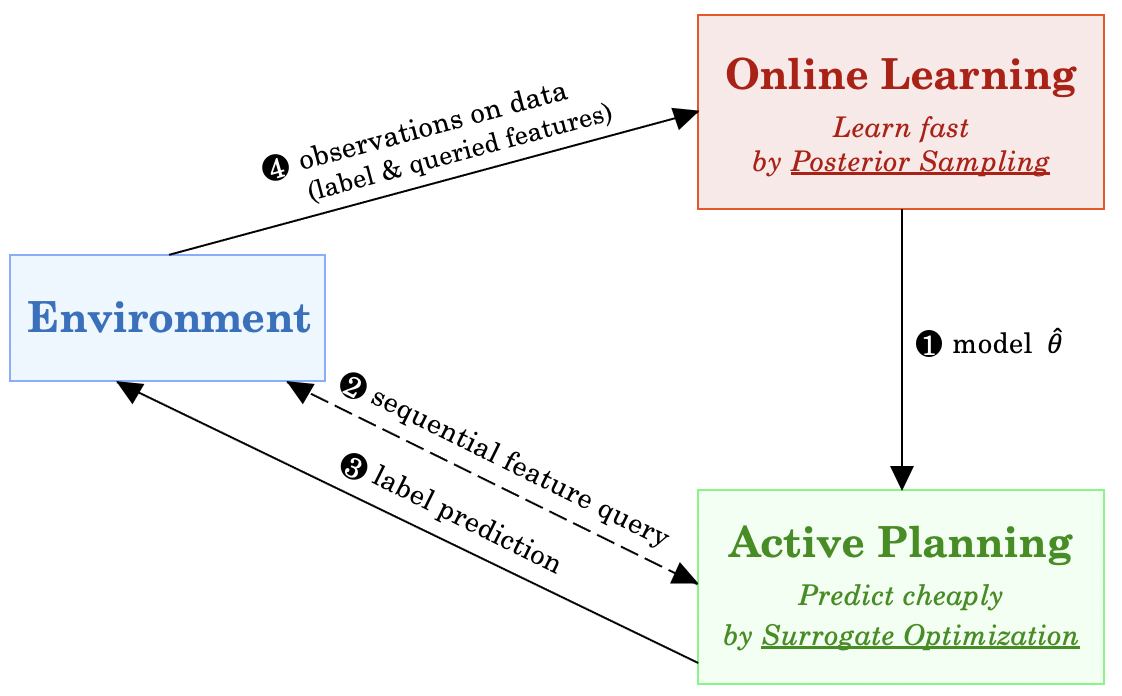}
\caption{An illustrative description for our proposed UFODT framework with an active planning oracle embedded in an online learning scheme. Details can be found in Algorithm~\ref{alg:onlinesketch}.}
\vspace{-0.20in}
\label{fig:framework}
\end{figure}

For the online learning part (in the outer loop), we employ posterior sampling \cite{osband2017posterior} to learn the online decision tree model. In addition to the established regret guarantee in the canonical online learning setting, another advantage is that posterior sampling can effectively leverage data-dependent prior knowledge, which the classical online decision tree models often fail to capture.

The active planning oracle adopts a decision-theoretic design: we aim to optimize the \emph{utility of the features}, (informally) defined as the expected prediction error for the incoming data point in the data stream, should we observe the value of the chosen features. 
In order to \textit{efficiently} optimize the utility of features, we consider adaptive surrogate objective functions following the key insight of \cite{nipsnoisy2013} to sequentially query features and their values, which enables us to predict accurately with low cost. In particular, the sequential feature query based on our surrogate objective 
is a natural analogy to 
the information gain node splitting criterion in the classical decision tree literature.


}
\yeol{
Our contributions are summarized as follows:

\begin{itemize}\denselist
\item We introduce a novel setting of online decision tree learning where the learner does not have \textit{a priori} access to both the feature values and the labels, and propose an efficient algorithmic framework for constructing online decision trees in a cost-effective manner. 
\item Our framework consists of several novel algorithmic contributions, including a novel surrogate objective as node splitting criterion (section 4.2), an extension to efficiently handle real-valued feature values 
(section 4.3), a variant to handle concept drift in streaming data (section 4.4), and an online feature selection scheme that further reduces the computational cost (appendix).
    \item We provide a rigorous theoretical analysis and justification of our algorithm, 
    in terms of both a prior-independent and a prior-dependent regret bound.
    \item We perform extensive experiments on diverse real-world datasets to verify that our framework is able to achieve competitive (or even better) accuracy with much lower cost, compared to baseline models.
\end{itemize}
    
}


\section{Related Work}\label{sec:lit}
\begin{table*}[t]
\scriptsize
\center
\resizebox{\width}{\height}{%
\begin{tabular}{@{}ccc@{}}
\toprule
Algorithms                        & Learning setting                           & Cost of feature query   \\ \midrule
Online decision tree (e.g., \cite{das2019learn}) & Semi-online (labels received in foresight)          & No                        \\
Bandit tree (e.g., \cite{feraud2016random})                   & Online (no labels but rewards received)             & No     \vspace{-0.4mm}                         \\
Active feature acquisition (e.g., \cite{shim2018joint})    & Offline or semi-online (labels received in foresight)                    & Yes                        \\
UFODT (\textbf{Ours})                                & {\textbf{Online}} (labels received only in hindsight) & {\textbf{Yes}}                \\ \bottomrule

\end{tabular}
}
\caption{Comparison of our framework with existing ODT literature. We make the first concrete attempt on taking feature acquisition cost into online decision tree learning problems.}
\vspace{-5mm}
\label{table:comparison}
\end{table*}
\yeol{
\textbf{Online Decision Tree.}
Traditional models consider to build online decision tree (ODT) incrementally. \cite{domingos2000mining} first propose VFDT to learn decision tree from streaming data, and use Hoeffding bound to guarantee the model performance; VFDT later becomes the de facto baseline in this domain. \cite{hulten2001mining} propose a variant to handle concept drift, but the construction for the tree growing process is complicated to implement. \cite{manapragada2018extremely} design Hoeffding Anytime Tree as an improvement for VFDT. \cite{das2019learn} suggest a bootstrap strategy to enhance the memory efficiency of VFDT. It is important to note that all those models are not fully online, nor do they consider feature query cost. Another line of work considers applying reinforcement learning to build decision trees online \cite{garlapati2015reinforcement,blake2018reinforcement}. However, these works do not have any theoretical guarantees, nor do they utilize prior knowledge (e.g., on the underlying state transition distributions).
}

\noindent
\textbf{Posterior sampling based online learning.} \yeol{Posterior sampling, also called Thompson sampling, is first proposed in \cite{thompson1933likelihood} to solve bandit problems in clinical trials, and the central idea is to select actions according to its posterior probability to be optimal. It later becomes an important policy in online learning problems, showing excellent performance empirically and theoretically. \cite{osband2013more,agrawal2017posterior,fan2021model} apply posterior sampling and prove its efficiency in reinforcement learning; this line of work is generally referred to as PSRL. Our work is closest to \cite{uaivoi2017}, which adapts PSRL to solve online information acquisition problems; however, in contrast to our work, \cite{uaivoi2017} 
consider a more constrained application domain of interactive troubleshooting,
and fail to tackle concept drift which is often crucial in data-streaming scenario; in addition, it tackles the hypothesis space in more restricted ways. 
}

\noindent
\textbf{Active feature acquisition.} \yeol{
The line of work on active feature acquisition (AFA) seeks to solve specific tasks like classification when data features are acquirable at a cost.  \cite{kapoor2009breaking} consider the restrictive setting where features and labels are boolean; \cite{bilgic2007voila} propose a decision-theoretic strategy with Bayesian networks to calculate the value of information of features; \cite{shim2018joint} suggest a joint framework to dynamically train the classifier while acquiring features. Conducting AFA online has not been discussed until recently, for example, \cite{beyer2020active} apply classical information acquisition techniques like \emph{information gain} to handle streaming data with missing features.
}

\noindent
\textbf{Adaptive information acquisition for decision making.} As fundamentals of sequential decision making, the goal of those works is to design an adaptive policy to identify an unknown target (i.e., a decision) by sequentially picking tests and observing outcomes (i.e., acquiring information). There are well-known greedy heuristics for the adaptive policy such as Information Gain (IG)~\cite{dasgupta2005analysis} and Uncertainty Sampling (US) which greedily maximize the uncertainty reduction (over different random variables). Recently, researchers propose to optimize w.r.t. submodular surrogates, e.g., $\text{EC}^{2}$~\cite{nipsnoisy2013}, HEC~\cite{javdani2014near}, ECED~\cite{chen2017near}, which are proven to have near-optimal performance with low information acquisition cost. The above-mentioned policies naturally fit into our problem and can all be used 
in the active planning phase. 

In Table \ref{table:comparison}, we provide a comparison of our work with existing ODT literature.

\section{Problem Formulation}\label{sec:formulation}
\subsection{Efficient Online Decision Tree Learning}
\label{subsec:tree}
\yeol{
Simply put, our task is to predict labels (classes) of streaming data points by building a decision tree online with low feature acquisition cost. At each epoch (time) $t$, we receive a data point $\vx ^t$, whose feature values and label are unknown. To make prediction for $\vx ^t$, we can gather information by querying feature values; each query incurs a cost. The label of $\vx ^t$ will only be revealed at the end of each epoch after we make the prediction.

Formally, let $\vx = (x_1, x_2, \ldots, x_n)$ be the \textit{data point} with $n$ features. Let $X_i \in \mathcal{X} \triangleq \{0, 1\}$ denotes the random variable for the \textit{feature value} of the $i$th feature\footnote{For simplicity, in this section we assume features are binary. Our setting is extended to multicagorical or continuous feature cases in Section~\ref{sec:continuous}, where more details can be found in appendix.}, and let $Y_j \in \mathcal{Y} \triangleq \{y_{1},  y_2, \ldots, y_{m}\}$ be the random variable for the \textit{label} of the data point. The superscript $t$ denotes that the data is received at epoch $t$. We adopt the common naïve Bayes assumption to model underlying probabilistic strucutre: $\mathbb{P}\left[Y_j, X_{1}, \ldots, X_{n}\right]=\mathbb{P}[Y_j] \prod_{i=1}^{n} \mathbb{P}\left[X_{i} \mid Y_j \right]$, i.e., features are conditionally independent given the class. Since we are in the online setting, we assume that the \textit{joint distribution} $\mathbb{P}\left[Y, X_{1}, \ldots, X_{n}\right]$ is initially unknown (though we may have prior knowledge on that) and needs to be learned via our online interactions.

We define $\theta _{ij} \triangleq \mathbb{P}\left[X_{i} = 1\mid Y_j \right]$, and assume that $\theta$ is follows a Beta distribution, $\mathrm{Beta}(\alpha _{ij}, \beta _{ij})$. We use $\boldtheta=\left[\theta_{i j}\right]_{n \times m}$ to denote the probabilistic table for the data distribution and assume $\boldtheta \sim \mathrm{Beta}(\boldalpha, \boldbeta)$.
Under the above probabilistic model, each query on a feature value will provide some information about $Y$. We define the set of \textit{queries} as $\gQ \triangleq \{1, 2, \ldots, n \}$, and a query $i \in \mathcal{Q}$ will reveal the value of the $i$th feature. We define \textit{cost} of the feature query as $c: \mathcal{Q} \rightarrow \mathbb{R}_{\geq 0}$. 
Upon information gathered from feature query, we make a prediction. We define the loss of our prediction as $l: \mathcal{Y} \times \mathcal{Y} \rightarrow \mathbb{R}_{+}$. Our goal is to reach low prediction loss with low query cost on data stream.

We additionally define $H=\left[X_{1}, \ldots, X_{n}\right]$ as the random variable for the \textit{hypothesis} of a data point. Thus, each hypothesis $h$ corresponds to a full realization of the outcome of all queries in $\mathcal{Q}$. Let $h \in \mathcal{H} \triangleq \{0, 1\}^{n}$. Importantly, the set $\mathcal{H}$ can be partitioned into $m$ disjoint \textit{decision regions}, that each class in $\mathcal{Y}$ corresponds to a decision region. Later, we may use the term ``decision region'' to implicitly refer ``label'' or ``class''.
}


Under this construction, our goal then becomes building a decision tree which identifies the \textit{decision region} for each data point arriving to us. Such identification of decision region should be done with low \emph{cost}. \yeol{This enables a decision-theoretic perspective as follows.}



\subsection{Utility of Features}
\yeol{
At each epoch, we perform a set of queris $\mathcal{F} \subseteq \mathcal{Q}$, and let the outcome vector be $\vx _{\calF}$, which can be conceived as a partial realization for the hypothesis $h$ of $\vx$.

Let $y$ be our label prediction, and denote its associated loss w.r.t. the true data label $y_ \text{true}$ as $l$. We can then naturally define the \emph{utility of $y$} as $u \triangleq -l$ and the \emph{conditional expected utility} of $y$ upon observing $\vx _\calF$ as $U\left(y \mid \vx _{\calF} \right) \triangleq \mathbb{E}_{y_{\text{true}}}\left[u(y_{\text{true}}, y) \mid \vx _{\calF}\right]$. Note that we could define the utility similarly upon $h$, since $h$ is the full realization of features.


\begin{definition}\label{thm:uof}
    \textit{(Utility of features $\vx _{\calF}$)}
    \begin{align}
        \sU \left(\vx_{\mathcal{F}}\right) &\triangleq \max _{y \in \mathcal{Y}} \ U\left(y \mid \vx _{\calF} \right). \nonumber
    \end{align}
\end{definition}

Here, $\sU (\vx _{\calF})$ represents the maximal expected utility achievable given $\vx _{\calF}$. This formulation is similar to the value of information~\cite{howard1966information}, and can connect with the generalization error of the classical empirical risk minimization framework~\cite{vapnik1992principles}, however, what we would like to emphasize here is that such utility relies on the \emph{partial realization of features}, that we seek to find the cheapest query set $\calF$ to achieve the maximal utility.

%

We then define the decision region for $y$ as the set of hypotheses for which $y$ is the optimal label prediction:
\begin{definition}\label{thm:region}
    \textit{(Decision region for $y$)}
    \begin{align}
        \calR _{y} \triangleq\{h: U(y \mid h)=\sU (h)\}. \nonumber
    \end{align}
\end{definition}

Directly optimizing $\sU (\vx _{\calF})$ is usually intractable, and greedy heuristics may fail or be costly. 
In the next, we will show how we may use a \emph{surrogate objective} of $\sU (\vx _{\calF})$ by the notion of decision regions to achieve near optimal query planning.
}

\section{Proposed Framework}\label{sec:algorithm}

\yeol{
We now present our UFODT framework for efficient online decision tree learning. The high-level structure is presented in Figure~\ref{fig:framework}, which can be conceived as an \textit{active planning oracle} nested within on \textit{online learning model}. We use the posterior sampling strategy for the online learning model, and a surrogate optimization algorithm on utility of features for the active planning oracle.

}



\subsection{Online Learning by Posterior Sampling}
\label{sec:online}

\yeol{
Assume that we have access to the prior of the environment parameter $\boldtheta$. Firstly, at the beginning of each epoch $t$, we sample $\boldtheta ^{t}$ from the (posterior) distribution of $\boldtheta$. Then, we run a adaptive policy which sequentially queries features (i.e., splitting nodes) of $\vx ^{t}$, in order to optimize some objectives (e.g., information gain, utility of features, etc.); importantly, such a policy can be conceived as an offline oracle, as its planning is fixed upon each sampled $\boldtheta ^{t}$. The policy will suggest a label prediction for $\vx ^{t}$. Finally, the true label for $\vx ^{t}$ is revealed, and is then used to update the posterior of $\boldtheta$ together with the query observation $\vx _{\calF}$. The pseudo-code is provided in \algref{alg:onlinesketch}.
}
\noindent
\scalebox{0.85}{
\begin{minipage}{1.15\columnwidth}
\begin{algorithm}[H]
  \caption{Online Decision Tree Learning}
  \label{alg:onlinesketch}
  \hspace*{\algorithmicindent} \textbf{Input}: Prior $\sP (Y)$ and $\sP (\boldtheta)$. 
  \vspace{+0.05in}
  \begin{algorithmic}[1]
        \For{$t=1, 2, \ldots, T$}
        \State Sample $\boldtheta ^{t} \sim \mathrm{Beta}(\boldalpha ^{t-1}, \boldbeta ^{t-1})$ and receive $\vx ^t$;
        \State Call Algorithm~\ref{alg:offlinesketch} with $\boldtheta _t$ to sequentially query features and predict the label (online);
        \State Observe $\vx ^t _{\calF}$ and true label $y_j ^{t}$;
        \State Call Algorithm~\ref{alg:stationary-ts} to obtain $\mathrm{Beta}(\boldalpha ^{t}, \boldbeta ^{t})$
        \EndFor

  \end{algorithmic}
\end{algorithm}
\end{minipage}
}

\noindent
\scalebox{0.85}{
\begin{minipage}{1.15\columnwidth}
\begin{algorithm}[H]
   \hspace*{\algorithmicindent} \textbf{Input}: $\vx ^t _{\calF}$; $y_{j} ^{t}$; $(\boldalpha ^{t-1}, \boldbeta ^{t-1})$. \vspace{+0.05in}
    \caption{Posterior Update}
    \label{alg:stationary-ts}
    \begin{algorithmic}[1]
          \For{each $(i, x _{i}) \in \vx _{\calF} ^t$}
          \If{$x_i = 1$}{
          $\alpha _{ij} ^{t}$ $\gets$ $\alpha _{ij} ^{t-1} + 1$

          }
          \Else
          {
          $\beta _{ij} ^{t}$ $\gets$ $\beta _{ij} ^{t-1} + 1$
          }
          \EndIf
          \EndFor
          \State $\textrm{Return $(\boldalpha ^{t}, \boldbeta ^{t})$}$
    \end{algorithmic}
\end{algorithm}
\end{minipage}
}

\subsection{Planning by Surrogate Optimization}
\label{sec:offline}
\noindent
\scalebox{0.85}{
\begin{minipage}{1.15\columnwidth}
\begin{algorithm}[H]
  \caption{Planning by Surrogate Optimization}
  \label{alg:offlinesketch}
  \hspace*{\algorithmicindent} \textbf{Input}: Prior $\sP (Y)$ and $\boldtheta$. \vspace{+0.05in}
\begin{algorithmic}[1]
        \State Sample hypotheses by calling Algorithm \ref{alg:sampling}
        \State $\calO=\emptyset$
        \While{stopping condition for EC$^2$ not reached}
            \State $\textrm{Use } \text{EC}^2 \textrm{ to determine next feature $i \in \gQ$ }$
            \State $\textrm{Query feature $i$}$
            \State $\textrm{Add $(i, x_i)$ to $\calO$}$
            \State $\textrm{Update $\sP(h \mid \calO)$ based on $\sP (Y)$ and $\boldtheta$}$
      \EndWhile
      \State $\textrm{Return the decision region $y$}$
\end{algorithmic}
\end{algorithm}
\end{minipage}
}\\

In the planning phase, we seek to optimize an objective of $\sU (\vx _{\calF})$ given the sampled environment. We propose to optimize for the \emph{surrogate objective} of $\sU (\vx _{\calF})$. Specifically, we focus on the EC$^{2}$ algorithm~\cite{nipsnoisy2013}, which uses the equivalence class edge cut as the surrogate objective of $\sU (\vx _{\calF})$.
Importantly, this surrogate objective function is adaptive submodular, and hence a greedy algorithm could attain a near optimal solution, allowing us to make accurate prediction with low query and computational cost. 

In EC$^{2}$, we define a weighted graph $G=(\mathcal{H}, E)$, where $E \triangleq \bigcup_{y \neq y^{\prime}}\left\{\left\{h, h^{\prime}\right\}: h \in \mathcal{R}_{y}, h^{\prime} \in \mathcal{R}_{y^{\prime}}\right\}$, denoting the pairs of hypotheses with different labels. The weight of each edge is $w\left(\left\{h, h^{\prime}\right\}\right) \triangleq \mathbb{P}(h) \cdot \mathbb{P}\left(h^{\prime}\right)$. Specifically, $\mathbb{P}(h)$ can be conceived as posterior distribution upon query of existing feature values. We define the weight of a set of edges as
$w\left(E^{\prime}\right) \triangleq \sum_{\left\{h, h^{\prime}\right\} \in E^{\prime}} w\left(\left\{h, h^{\prime}\right\}\right)$. Therefore, performing a \textit{feature query} is considered as \textit{cutting an edge}, which can also be conceived as removing inconsistent hypotheses with all their associated edges. We thus have the edge set $E\left(x_{i}\right)$ cut after observing the outcome of a feature query $x_i$: $E\left(x_{i}\right) \triangleq\left\{\left\{h, h^{\prime}\right\} \in E: \mathbb{P}\left[x_{i} \mid h\right]=0 \vee \mathbb{P}\left[x_{i} \mid h^{\prime}\right]=0\right\}$.
Based on the graph $G$, we formally define the EC$^2$ objective as $f_{E C^{2}}\left(\mathbf{x}_{\mathcal{F}}\right) \triangleq w\left(\bigcup_{v \in \mathcal{F}} E\left(x_{v}\right)\right), \nonumber$
and the score of feature query is defined as 
$
\Delta_{E C^{2}}\left(u \mid \mathbf{x}_{\mathcal{F}}\right) \triangleq \mathbb{E}_{x_u|\mathbf{x}_{\mathcal{F}}}\left[f_{E C^{2}}\left(\mathbf{x}_{\mathcal{F} \cup \{u\}}\right) - f_{E C^{2}}\left(\mathbf{x}_{\mathcal{F}}\right)\right]. \nonumber
$
The policy $\pi _{\text{EC}^{2}}$ will greedily query the feature which maximizes the gain cost ratio $\Delta_{E C^{2}}\left(v \mid \mathbf{x}_{\mathcal{F}}\right) / c(v)$ and stops when only one decision region exists. We present the algorithm in Algorithm \ref{alg:offlinesketch}. Note that the EC$^2$ objective in the line 3 and line 4 can be flexibly replaced by other active information acquisition functions like Information Gain (IG) and Uncertainty Sampling (US), which we elaborate in the appendix.

\paragraph{Hypothesis Sampling Procedure.}
Information acquisition methods such as Information Gain and Uncertainty Sampling, and also EC$^2$ require enumeration of hypothesis space, which can be computationally challenging. 
Thereby, to reduce the number of hypotheses, we use a sampling procedure sketched in Algorithm \ref{alg:sampling}. In this algorithm we first sample a decision region using the prior distribution over the classes and then we exploit the current estimate of $\boldtheta$ to build a new sample.
\noindent
\scalebox{0.85}{
\begin{minipage}{1.15\columnwidth}
\begin{algorithm}[H]
  \caption{Hypotheses Sampling}
  \label{alg:sampling}
  \hspace*{\algorithmicindent} \textbf{Input}: Prior $\sP (Y)$ and $\boldtheta$. \vspace{+0.05in}
  \begin{algorithmic}[1]
        \State $\tilde{\calH}$ $\gets$ $ \emptyset$
        \State \textrm{Sample decision regions from $\sP (Y)$}
        \For{each sampled decision region $j$}
        \State $h$ $\gets$ $ \emptyset$
            \For{each $i \in \gQ$}
                \State \textrm{Sample $X_i \sim $ Ber($\theta_{ij}$) and add to $h$}
            \EndFor
            \State $\tilde{\calH} = \tilde{\calH} \cup h$
        \EndFor
        \State $\textrm{Return $\tilde{\calH}$}$
 \end{algorithmic}
\end{algorithm}
\end{minipage}
}
\noindent
\scalebox{0.85}{
\begin{minipage}{1.15\columnwidth}
\begin{algorithm}[H]
\caption{Threshold selection}\label{alg:exp3_threshold_selection}
\hspace*{\algorithmicindent}\textbf{Input}: $\eta$
\begin{algorithmic}[1]
\State $S_{ik}^{(0)} \gets 0$ for all features $i$ and thresholds $k$
\For{$t=1, 2, \ldots, T$}
\For{each feature $i$}
\If{feature $i$ can be queried}
    \State Calculate the threshold sampling distribution:
    \begin{center}
        $\Pi_{ti}(k) = \frac{\exp(\eta S_{ik}^{(t-1)})}{\sum_{k'}\exp(\eta S_{ik'}^{(t-1)})}$
    \end{center}
    \State Sample threshold $B_{ti} \stackrel{}{\sim} \Pi_{ti}$ and observe gain $\Delta_{ti}$
    \State Calculate  $S_{ik}^{(t)}$: 
    \begin{center}
        $S_{ik}^{(t)} = S_{ik}^{(t-1)} + \frac{\mathbbm{1}\{B_{ti}=k\}\Delta_{ti}}{\Pi_{ti}(k)}, \text{ for all } k$
    \end{center}
\EndIf
\EndFor
\EndFor
\end{algorithmic}
\end{algorithm}
\end{minipage}
}
\subsection{Handling Continuous Features}
\label{sec:continuous}



One way to extend our framework for handling continuous data is to ``binarize" real-valued features.
In particular, for each feature we consider $K$ different thresholds for binarization, and in each training time step, we select the threshold that maximizes the gain based on the information acquisition function (e.g., $\Delta_{E C^{2}}$). By collecting the history for each threshold, we can easily calculate the posterior distribution of the parameters associated to the binary feature corresponding to that threshold. We provide the details in the Appendix \ref{sec:continuous_featurs_details}. This naive way of \emph{exhaustively} searching for the best threshold causes a significant computational running time in each time step. Thereby, we propose a more efficient algorithm for \emph{learning} the best discretization for each feature. 

\paragraph{Learning discretizations for continous features.} We model the threshold selection process for each feature as an \emph{adversarial bandit problem} \cite{auer2002nonstochastic} with arms and rewards being the thresholds and gains, respectively. Let $\Pi_{ti}: [K] \rightarrow \mathbb{R}_{\geq 0}$ ($\sum_{k \in [K]} \Pi_{ti}(k)=1$) be the probability distribution according to which we select the binarization threshold for feature $i$ at time step $t$. Then, we do threshold selection and update $\Pi_{ti}$ with the procedure sketched in Algorithm \ref{alg:exp3_threshold_selection} (adapted from the Exp3 algorithm \cite{auer2002nonstochastic}). $S_{ik}^{(t)}$ is the sum of estimated gains for the $k$-th threshold  of the  $i$-th feature until time $t$. In time $t$, we use a threshold sampled from $\Pi_{ti}$ and observe the gain for that threshold. Then we calculate $S_{ik}^{(t)}$ based on $S_{ik}^{(t-1)}$ and the observed gain. Specifically, for each feature, we add unbiased estimates of gains ($\frac{\mathbbm{1}\{B_{ti}=k\}\Delta_{ti}}{\Pi_{ti}(k)}$) for different thresholds to the previous sum of gains.

 \subsection{Handling Concept Drift}
\label{sec:drift}
\yeol{Concept drift is a crucial problem in streaming scenarios, where the dependency of features on the data label is changing over time. Classical ODTs use complicated updating criteria to handle concept drift~\cite{hulten2001mining}. Thanks to our posterior sampling scheme, we are able to adopt an exceptionally easy solution to tackle the concept drift problem, by simply adding two lines of code upon Algorithm~\ref{alg:stationary-ts}, which is shown in Algorithm~\ref{alg:nonstationary-ts}.} \yeol{This inspiration comes from non-stationary posterior sampling~\cite{russo2017tutorial}. The central idea is that we need to keep exploring in order to learn the time-varying concept. This technique encourages exploration by adding a discount parameter $\gamma$ for the history, and injecting a random distribution $\mathrm{Beta}(\bar{\boldalpha}, \bar{\boldbeta})$ to increase uncertainty.}

\noindent
\scalebox{0.85}{
\begin{minipage}{1.15\columnwidth}
\begin{algorithm}[H]
  \hspace*{\algorithmicindent} \textbf{Input}: $\vx ^t _{\calF}$; $y_{j} ^{t}$; $(\boldalpha ^{t-1}, \boldbeta ^{t-1})$; $\gamma$; $(\bar{\boldalpha}, \bar{\boldbeta})$. \vspace{+0.05in}
  \caption{Handling Concept Drift}
  \label{alg:nonstationary-ts}
  \begin{algorithmic}[1]
        \State $\boldalpha ^{t} \gets (1 - \gamma) \boldalpha ^{t-1} + \gamma \bar{\boldalpha}$
        \State $\boldbeta ^{t} \gets (1 - \gamma) \boldbeta ^{t-1} + \gamma \bar{\boldbeta}$
        \For{each $(i, x _{i}) \in \vx _{\calF} ^t$}
            \If{$x_i = 1$} \State $\alpha _{ij} ^{t} \gets \alpha _{ij} ^{t-1} + 1$
            \Else
                \State $\beta _{ij} ^{t} \gets \beta _{ij} ^{t-1} + 1$
            \EndIf
        \EndFor
        \State $\textrm{Return $(\boldalpha ^{t}, \boldbeta ^{t})$}$
 \end{algorithmic}
\end{algorithm}
\end{minipage}
}
\section{Theoretical Analysis}
\label{bound-analysis}
In this section, we discuss the bound of the expected regret for our fully online framework (Sections \ref{sec:online} and \ref{sec:offline}). Here we focus on the EC\textsuperscript{2} objective function due to its theoretical guarantees. 

Let $\sU (\pi) \triangleq \mathbb{E}_{h}\left[ \max _{y \in \mathcal{Y}} \mathbb{E}_{y _{\text{true}}}[u(y_{\text{true}}, y) \mid \calS (\pi, h)] \right]$ be the expected utility of features achieved by a policy $\pi$; here, $\calS (\pi, h)$ represents the set of features and their values queried by policy $\pi$ upon a hypothesis $h$.
As proved by \cite{golovin2011adaptive}, by the submodularity of the $\text{EC}^{2}$ objective function, $\pi ^{\text{EC}^{2}}$ is able to achieve the same utility as the optimal policy $\pi ^{*}$ does under a same environment $\boldtheta$, with at most $\left(2 \ln \left(1 / p_{\min }\right)+1\right) \cdot c_{\pi ^{*}}$ query cost, where $p_{\min}$ denotes the minimal probability of a hypothesis $h$ across environments and $c_{\pi ^{*}}$ represents the cost of the optimal policy.

\begin{definition}\label{immediat-regret}
     Let $\boldtheta ^{\star}$ denote the true environment, and let $\boldtheta ^{t}$ denote the sampled environment at epoch $t$ as in line 2 of Algorithm~\ref{alg:onlinesketch}. Let $\pi ^{*} _{\boldtheta ^{\star}}$ denote the optimal policy for $\boldtheta ^{\star}$, and $\pi ^{\operatorname{EC}^{2}} _{\boldtheta ^{t}}$ denote the policy with $\operatorname{EC}^{2}$ as in Section~\ref{sec:offline}. We define the immediate regret at epoch $t$ for Algorithm~\ref{alg:onlinesketch} with the $\operatorname{EC}^{2}$ objective as:
    \begin{align}
        \Delta ^{t} \triangleq \sU (\pi ^{*} _{\boldtheta ^{\star}}) - \sU (\pi ^{\operatorname{EC}^{2}} _{\boldtheta ^{t}}). \nonumber
    \end{align}
    We then define the total regret at epoch $T$ as:
    \begin{align}
        \text{Regret}(T) = \sum _{t=1} ^{T} \Delta ^{t}.  \nonumber
    \end{align}
\end{definition}

\subsection{Prior-Independent Regret Bound}\label{psrl-bound}
Based on the result of \cite{osband2013more}, we have the following regret bound for Algorithm~\ref{alg:onlinesketch}:

\begin{theorem}\label{exp-regret}
    (Prior-independent regret bound) Let $L=\left(2 \ln \left(1 / p_{\min }\right)+1\right) \cdot c_{\pi ^{*}}$ denote the worst-case cost (i.e., number of queries) for Algorithm~\ref{alg:onlinesketch} with the $\operatorname{EC}^{2}$ objective in any epoch, where $p_{\min}$ denotes the minimal probability of a hypothesis $h$ across environments and $c_{\pi ^{*}}$ represents the cost of the optimal policy. Let $S$ be the number of possible realizations of $L$ queries and $n$ be the total number of features. Assume that the sampling of decision regions by Algorithm~\ref{alg:sampling} is sufficient, such that all hypotheses with non-zero probability in the hypothesis space are enumerated. The expected total regret at epoch $T$ for Algorithm~\ref{alg:onlinesketch} with the $\operatorname{EC}^{2}$ objective is:
    \begin{align}
    \mathbb{E}[\operatorname{Regret}(T)]=O(L S \sqrt{n L T  \log (S n L T )}).\nonumber
    \end{align}
\end{theorem}
We provide the proof in Appendix \ref{proof-exp-regret}.


The above regret bound depends on the worst-case cost $L$, which could potentially be huge, and the bound is also prior-independent such that the benefit of a good prior knowledge by the posterior sampling scheme is not reflected (as we illustrate in the experiments of Appendix~\ref{sec:priorvsutility} on the impact of priors). In the appendix, we in addition provide a prior-dependent bound based on the results from \cite{russo2016information} and \cite{lu2021reinforcement}.

\begin{figure*}[h!]
    \centering
    \includegraphics[width=1.0\textwidth]{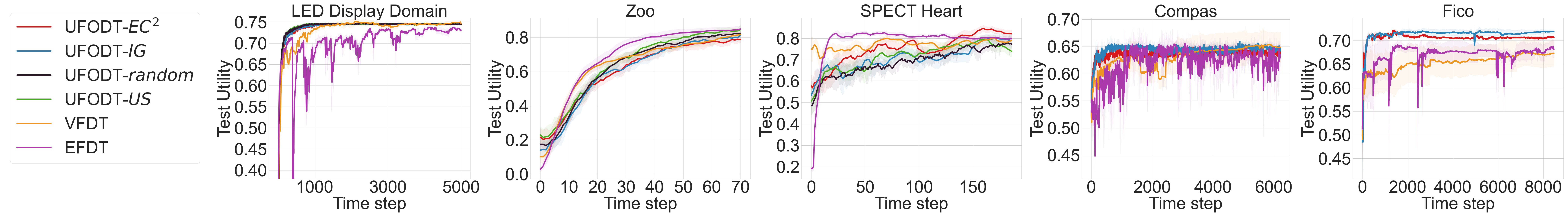}
    \caption{Test utilities during training:
    UFODT reaches test utilities comparable with those from VFDT and EFDT but with significantly lower costs. UFODT performs even better than VFDT and EFDT on LED, Heart and Fico datasets.}
    \label{fig:UCI_test_utility_in_progress}
\end{figure*}
\begin{figure*}[h!]
    \centering
    \includegraphics[width=1.0\textwidth]{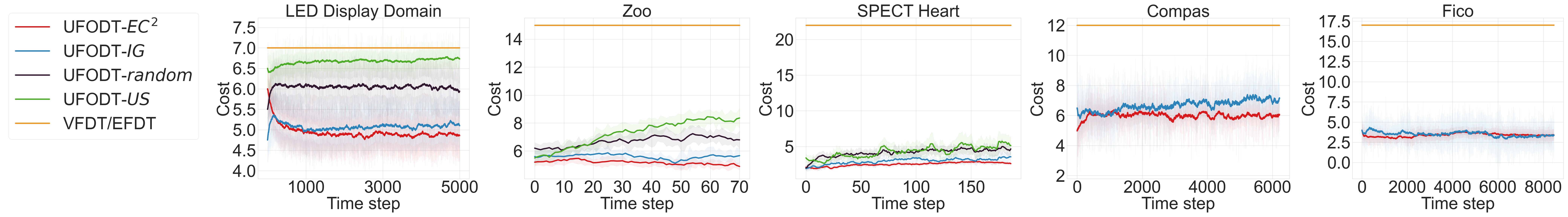}
    \caption{Querying costs during training:
    UFODT-EC\textsuperscript{2} yields the lowest cost during training for all datasets. The cost of our framework is significantly lower than the VFDT and EFDT algorithms which require all feature values during training steps.
    }
    \vspace{-2mm}
    \label{fig:UCI_cost_in_progress}
\end{figure*}

\section{Experimental Results}
\label{sec:experiments}
We now empirically validate our framework with real-world datasets. Unless otherwise specified, we assume that we have a uniform prior on $\boldtheta$ for each dataset initially. We evaluate the methods introduced in Section \ref{sec:algorithm} from  different aspects. We compute the average number of queries per online session to compare the costs of algorithms. We also evaluate the generalization power of classifiers via holdout test sets. Additionally, (in the appendix \ref{sec:additional_experiments}) we measure the prediction performance on training sets during learning together with various other aspects of our framework.

\noindent
\paragraph{Datasets.} We have used three 
 stationary datasets in our experiments that are standard binary classification datasets taken from UCI repository~\cite{Dua:2019}. Furthermore, we conduct experiments on the ProPublica recidivism (Compas) dataset \cite{recidivism2016} and the Fair Isaac (Fico) credit risk dataset \cite{dataFico} 
 as in \cite{hu2019optimal}. In Compas dataset, we predict the individuals arrested after two years of release, and in Fico we predict if an individual will default on a loan. 
 For concept drifting experiments, we adopt the non-stationary Stagger dataset~\cite{widmer1996learning,lobo2020synthetic}, where each data has three nominal attributes and the target concept will change abruptly at some point. For extensions to continuous features (as well as for feature selection in the appendix), we use Prima Indians Diabetes Dataset \cite{Prima:1988}, Breast Cancer Wisconsin Dataset \cite{wdbc1993} and Fetal Health Dataset \cite{fetal2000}.\\

\noindent
\paragraph{Algorithms.} The VFDT algorithm~\cite{domingos2000mining} is used as a classical baseline ODT model. We also compare our method with the EFDT algorithm \cite{manapragada2018extremely}. Within our proposed UFODT framework, we use four different information acquisition functions. The first one is EC\textsuperscript{2} which is proved to have near-optimal cost in offline planning. In addition to EC\textsuperscript{2}, we use Information Gain (IG) which selects the feature that maximizes the reduction of entropy in labels. Moreover, we conduct experiments with Uncertainty Sampling (US) which finds the feature causing the highest reduction in entropy of hypotheses. We also use random feature selection which randomizes the order of querying features. \vspace{-2mm}






\subsection{Experiments on Stationary Datasets}
\label{sec:UCI_datasets}

Figure \ref{fig:UCI_test_utility_in_progress} shows the average utility achieved in each training time step by different methods over a holdout test set. To compute the test utility for our UFODT framework (with different objectives) we use the current estimation of the parameters of the conditional distributions (i.e., the estimated $\boldtheta$ at time $t$) to obtain the test predictions. For VFDT and EFDT, we use the last version of the tree at time $t$. If a dataset is balanced we use accuracy as the utility, whereas we use f-measure for imbalanced datasets. Figure \ref{fig:UCI_cost_in_progress} shows the average cost (i.e., the number of features queried) in each training time step for different algorithms\footnote{To make the results more clear, we do not show the results of UFODT-random and UFODT-US for Fico and Compas datasets as they have higher query costs and generally lower test performances compared to UFODT-IG and UFODT-EC\textsuperscript{2}.}. 
For all datasets, we observe that
our UFODT framework reaches a very competitive test utility during training with a much lower cost. UFODT-EC\textsuperscript{2} yields the lowest cost among other information acquisition functions which is compatible with the theoretical results. 
 As discussed earlier, VFDT and EFDT are costly (i.e., require access to all features) and not fully online (labels are known in advance during training), while our proposed framework is cost-effective and fully online. We observe that, with exceptionally lower cost, our framework with different algorithms still reaches comparable or even better test utilities, compared with VFDT and EFDT. The number of sampled hypotheses in UFODT are 95, 161, 34, 500 and 50 for LED Display Domain, Zoo, SPECT Heart, Compas and Fico respectively. 
 In Appendix \ref{sec:additional_experiments}, we investigate other aspects of our UFODT framework using these datasets. 
\begin{figure*}[t]
    \centering
    \hspace{-7mm}
    \subfigure[Cost]
    {
        \includegraphics[width=0.5\textwidth]{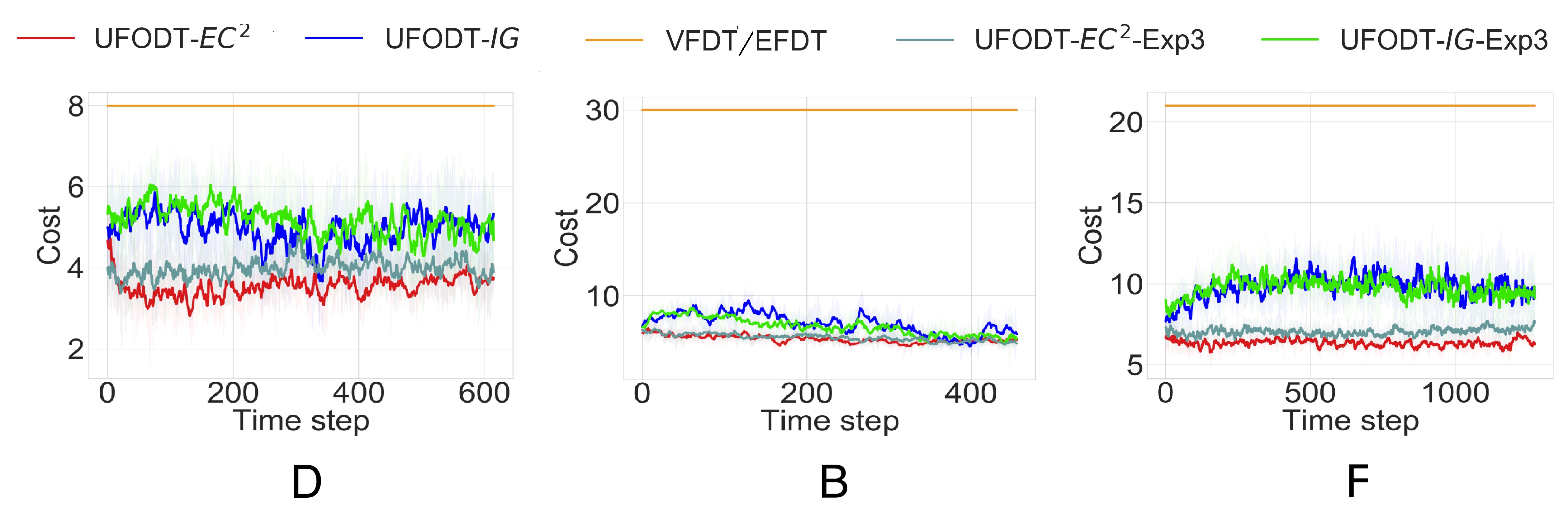}
        \label{fig:cost_in_progress_EXP3}
    }
    \subfigure[Test utility]
    {
        \includegraphics[width=0.5\textwidth]{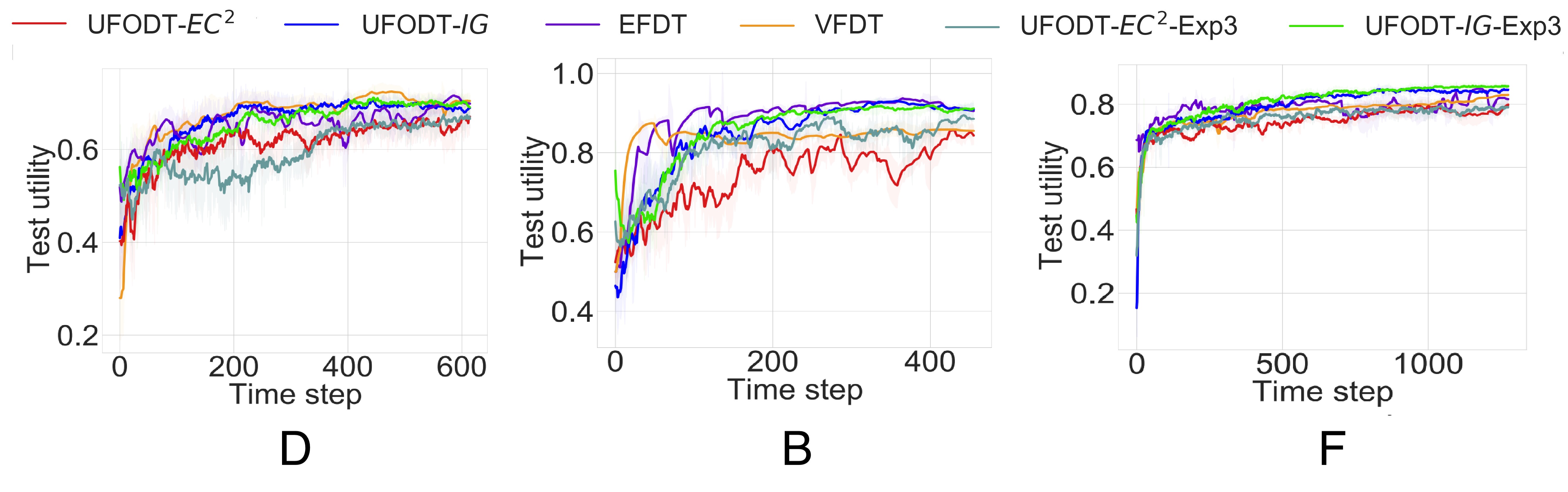}
        \label{fig:test_in_progress_EXP3}
    }
    \vspace{-3mm}
    \caption{The aversge cost (\ref{fig:cost_in_progress_EXP3}) and average test utility (\ref{fig:test_in_progress_EXP3}) during training for Prima Indians Diabetes (D), Breast Cancer (B) and Fetal Health (F) datasets. Our framework enjoys significantly lower cost while maintaining competitive prediction accuracy. Using Algorithm \ref{alg:exp3_threshold_selection} for learning thresholds yields competitive results with those of exhaustive search, but with significantly lower running time.}
    \vspace{-3mm}
\end{figure*}
\subsection{Extension to Continuous Features}
\label{sec:cont-exp}
We here experimentally investigate the effectiveness of our UFODT framework for non-binary real-valued features (see Section \ref{sec:continuous}). We conduct our experiments on three different classification datasets: i) Prima Indians Diabetes dataset (D) which has several medical predictor variables and two classes indicating the onset of diabetes mellitus, ii) Breast Cancer Wisconsin dataset (B), and iii) Fetal Health dataset (F). We employ the two methods discussed in Section \ref{sec:continuous} to handle continuous features, i.e., exhaustive search over thresholds and learning best the thresholds via Algorithm \ref{alg:exp3_threshold_selection} (shown by UFODT-\emph{criterion}-Exp3, e.g., UFODT-IG-Exp3). The number of sampled hypotheses in UFODT are 120, 500 and 500 for Diabetes, Breast Cancer and Fetal Health respectively. We use $\eta=0.01$ in Algorithm \ref{alg:exp3_threshold_selection}.


In Figure \ref{fig:cost_in_progress_EXP3}, 
we illustrate the average cost (number of queried features) in each time step of training. We compare the UFODT framework with EFDT and VFDT.
In Figure \ref{fig:test_in_progress_EXP3},
we show the test utility  during the training process (similar to Figure \ref{fig:UCI_test_utility_in_progress}). We repeat these experiments with 5 different random seeds and report the averaged results together with one-standard error. The results demonstrate again that our framework (UFODT-EC\textsuperscript{2} and UFODT-IG) has competitive prediction accuracy compared to  EFDT and VFDT while having an exceptionally lower feature acquisition cost. UFODT-EC$^{2}$ generally has the lowest cost for feature querying;  UFODT-IG yields a slightly higher cost than UFODT-EC$^{2}$ but reaches the same or even better test utilities than EFDT and VFDT algorithms. Moreover, we observe that the incurred querying costs and test utilities achieved by the threshold learning algorithm (Algorithm \ref{alg:exp3_threshold_selection}) is competitive with the exhaustive search method with a significantly lower running time. For instance, in case of UFODT-EC\textsuperscript{2}, we see that UFODT-EC\textsuperscript{2}-Exp3 yields even better test utility with slightly higher cost.
\subsection{Experiments with Concept Drift Dataset}
\label{sec:exp-drift}
\begin{table}[htb!]
\center
\resizebox{1.0\columnwidth}{0.9\height}{
\begin{tabular}{lccccc}
\toprule
\textbf{Method} & UFODT-EC${}^2$ & UFODT-IG & UFODT-US & EFDT
\\
&(Adaptive) & (Adaptive) & (Adaptive) &
\\
\midrule
\textbf{Cost $\downarrow$} & {$\mathbf{343.3} \pm 11.0$} & $350.6 \pm 5.1$ & $477.0 \pm 13.9$ & $720.0$  \\
\bottomrule
\end{tabular}
}
\caption{Average feature querying costs for Stagger dataset where UFODT-based methods incur significantly lower costs.}
\vspace{-1mm}
\label{table:cost_comparisons}
\end{table}
\vspace{-3mm}
\begin{figure}[htb!]
    \centering
     \includegraphics[width=0.4\textwidth]{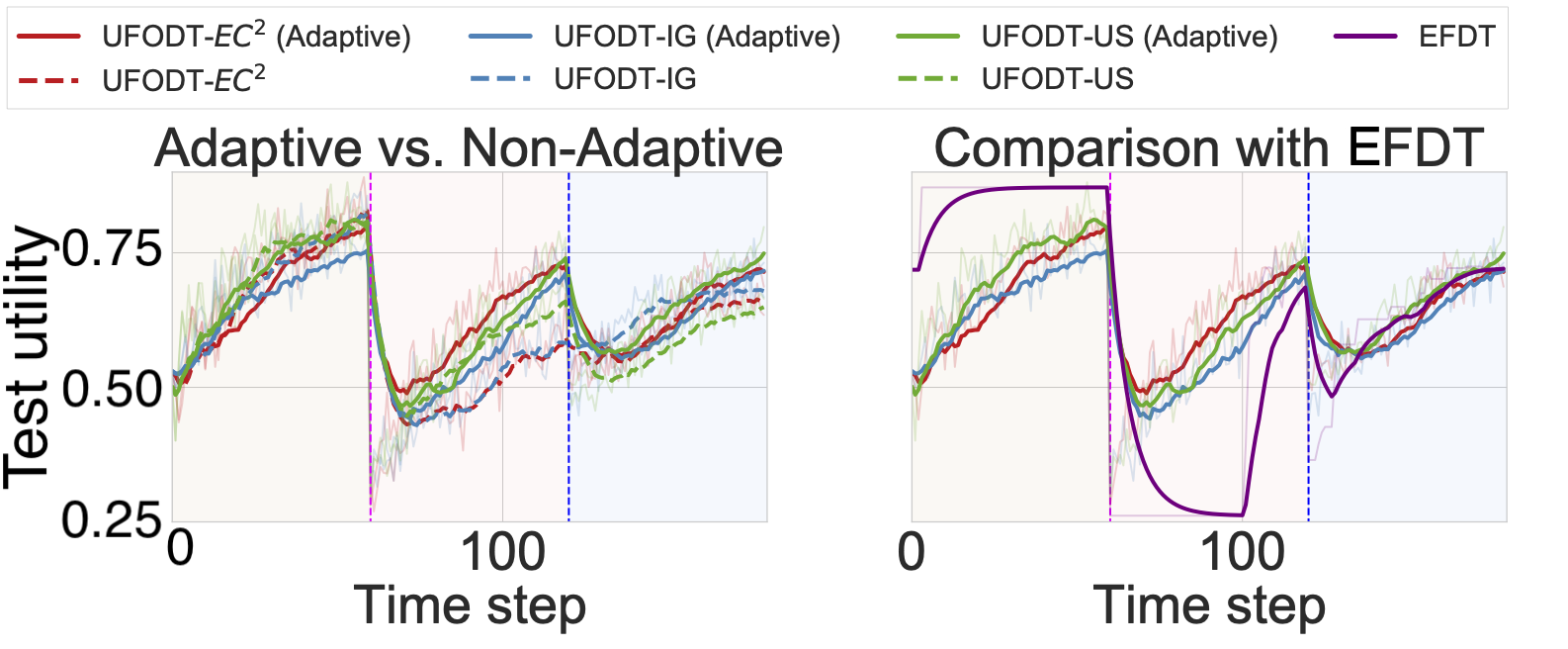}
    \caption{Time step vs. test utility on Stagger dataset. Each shaded area corresponds to one concept; the vertical dashed line shows when the drift happens. 
    }
    \vspace{-2mm}
     \label{fig:timevsutility}
\end{figure}

In this section, we demonstrate the effectiveness and flexibility of our framework under the concept drift setting. Here, the non-stationary nature of the online data imposes extra difficulty for online decision tree problems. To handle the concept drift, we adopt non-stationary posterior sampling (Algorithm~\ref{alg:nonstationary-ts}). We compare our proposed algorithm to EFDT, which is the SOTA baseline for solving concept drifting problem in online decision tree learning. 
To simulate the concept drift scenarios, we adopt the Stagger dataset. 
In this dataset, 
there are in total two concept drifting that happens abruptly at time steps $60$ and $120$. 
We repeat each experiment with $10$ random seeds and report the averaged results along with one-standard error. 

\noindent
\paragraph{Test utility vs. time step.}
In Figure~\ref{fig:timevsutility}, we report the results of UFODT-EC${}^2$, UFODT-IG and UFODT-US with both standard posterior sampling and non-stationary posterior sampling~(denoted with \emph{Adaptive}). For all of our methods, we adopt the uniform prior. We can observe that all the three methods with non-stationary posterior sampling can adapt to the abrupt concept drift much faster, and also achieve higher test utility. We also compare these three methods against EFDT in the right part of Figure \ref{fig:timevsutility}.
Though, initially, EFDT can achieve higher utility than our methods, it has a big drop in utility after both the first and second concept drifting. This demonstrates advantages of our methods in quickly adapting to new concepts over EFDT. In addition, we also report the averaged total number of feature queries as a cost measure in Table~\ref{table:cost_comparisons}. We observe that our approaches require significantly fewer feature queries~(or lower cost) but still achieve higher utility than EFDT.

\vspace{-3mm}
\section{Conclusion}\label{sec:conclusion}
We make the first concrete step towards learning decision trees online with incorporating feature acquisition cost. Within the proposed framework, to learn efficiently with less time, we utilize a posterior sampling scheme; to predict efficiently with lower cost, we employ various information acquisition objectives including a surrogate objective function on utility of features, enabling near-optimal feature acquisition cost with competitive prediction accuracy. 
Our framework also provides several novel algorithmic contributions including a simple and flexible solution to the concept drift problem, an extension to efficiently handle real-valued features and a computationally efficient online feature selection scheme. In general, our work opens a new and practical direction of online decision tree learning on cost-sensitive applications, and we demonstrate the great potential of active information acquisition strategies in such applications. 
\section*{Acknowledgements}
This work was partially supported by the Wallenberg AI, Autonomous Systems and Software Program (WASP) funded by the Knut and Alice Wallenberg Foundation. We would like to thank Chaoqi Wang and the anonymous reviewers for their constructive comments.
\bibliographystyle{named}
\bibliography{refs}
\clearpage
\newpage

\appendix
\thispagestyle{empty}
\section{Details of Other Active Information Acquisition Objectives}
In this section we provide the details of the other two active information acquisition functions which can be implemented within our framework. Unlike EC$^2$, the objective function of those approaches are not submodular, thus they may fail arbitrarily badly in certain cases, as illustrated in \cite{golovin2011adaptive}.


\paragraph{Information Gain (IG).} The high-level idea of IG is to greedily select the feature that achieves the maximum entropy reduction for the label. The score of each query is defined as:
\begin{align}
\Delta_{IG}\left(u \mid \mathbf{x}_{\mathcal{F}}\right) \triangleq \mathbb{H}(Y|\mathbf{x}_{\mathcal{F}}) -  \mathbb{E}_{x_u|\mathbf{x}_{\mathcal{F}}}[\mathbb{H}(Y|\mathbf{x}_{\mathcal{F} \cup \{u\}})]. \nonumber
\end{align}

One could calculate $\mathbb{H}(Y|\mathbf{x}_{\mathcal{F}})$ as following:
\begin{align}
\mathbb{H}(Y|\mathbf{x}_{\mathcal{F}}) = - \sum_{Y_j} \mathbb{P}[Y_j|\mathbf{x}_{\mathcal{F}}]\log\mathbb{P}[Y_j|\mathbf{x}_{\mathcal{F}}], \nonumber
\end{align}
and
\begin{align}
\mathbb{P}[Y_j|\mathbf{x}_{\mathcal{F}}] = \frac {\mathbb{P}[\mathbf{x}_{\mathcal{F}}|Y_j]\mathbb{P}[Y_j]}{\mathbb{P}[\mathbf{x}_{\mathcal{F}}]} = \frac {\prod_{i \in \mathcal{F}}\mathbb{P}[x_i|Y_j]\mathbb{P}[Y_j]}{\mathbb{P}[\mathbf{x}_{\mathcal{F}}]}, \nonumber
\end{align}
where
\begin{align}
\mathbb{P}[x_i|Y_j] = \theta_{ij}^{x_i}(1-\theta _{ij})^{(1-x_i)} \nonumber
\end{align}
and
\begin{align}
\mathbb{P}[\mathbf{x}_{\mathcal{F}}] = \sum_{Y_j}\mathbb{P}[Y_j]\mathbb{P}[\mathbf{x}_{\mathcal{F}}|Y_j] = \sum_{Y_j}(\mathbb{P}[Y_j]\prod_{i \in \mathcal{F}}\mathbb{P}[x_i|Y_j]). \nonumber
\end{align}

\paragraph{Uncertainty Sampling (US).} The high-level idea of US is to greedily select the feature that maximizes the reduction of entropy in the hypotheses space. Specifically, the score of each query is defined as:
\begin{align}
\Delta_{US}\left(u \mid \mathbf{x}_{\mathcal{F}}\right) \triangleq \mathbb{H}(H|\mathbf{x}_{\mathcal{F}}) -  \mathbb{E}_{x_u|\mathbf{x}_{\mathcal{F}}}[\mathbb{H}(H|\mathbf{x}_{\mathcal{F} \cup \{u\}})].\nonumber
\end{align}
The detailed calculation can be derived similarly as that of IG's.
\section{Prior-Dependent Regret Bound}

\begin{theorem}\label{info-regret}
    \textit{(Prior-dependent regret bound)} Let $\sH( \boldtheta ^{\star})$ denote the initial information entropy of the true environment $\boldtheta ^{\star}$, and $\bar{\Gamma}$ denote the maximal information ratio\footnote{We leave the exact analytical form of $\bar{\Gamma}$ as the future work.} of Algorithm~\ref{alg:onlinesketch} with the $\operatorname{EC}^{2}$ objective. The expected total regret at epoch $T$ for Algorithm~\ref{alg:onlinesketch} with the $\operatorname{EC}^{2}$ objective is:
    \begin{align}
    \mathbb{E}[\operatorname{Regret}(T)] \leq \sqrt{\bar{\Gamma} \sH( \boldtheta ^{\star}) T}.\nonumber
    \end{align}
\end{theorem}

As implied by this regret bound, a more informative prior will lead to smaller value of $\sH( \boldtheta ^{\star})$, hence a better bound; this also aligns with our observations in Figure~\ref{fig:priorvsutility} (in the appendix), showing that our framework has much more practicality and flexibility over traditional ODT models: our framework can effectively use prior knowledge.


This prior-dependent bound for posterior sampling is first proposed by \cite{russo2016information} for the multi-armed bandit problems. The core of the analysis is the \emph{information ratio}, which precisely captures the exploration-exploitation tradeoff of the policy at each time epoch.

This bound may potentially be ``better'' than the previous bound (Theorem \ref{exp-regret}) in terms of its dependence on the information ratio, and the dependence on the prior information (initial epistemic uncertainty) of the environment. To explain, firstly, the information ratio can be bounded by certain ``dimension'' of the problem (e.g., the feature dimension of linear bandits), which can be vastly smaller than the cardinality of action/state space; secondly, the initial epistemic uncertainty reflects our prior knowledge on the environment, whereas the previous regret bound cannot benefit from. We provide the proof of Theorem~\ref{info-regret} in Appendix \ref{proof-info-regret}.

\section{Proofs}
\subsection{Proof of Theorem~\ref{exp-regret}}
\label{proof-exp-regret}
The proof of Theorem~\ref{exp-regret} relies on the following lemma:
\begin{lemma}\label{psrl}
    \textit{(Theorem 1 of \cite{osband2013more})}
    Consider learning to optimize a random finite horizon $M = (\calB, \calA, R^{M}, P^{M}, L, \rho)$ in $T$ repeated time epochs, where $\calB$ denotes the state set with cardinality $S$, $\calA$ denotes the action set with cardinality $A$, $R^{M}$ denotes the reward function, $P^{M}_{i}(s' \mid s)$ represents the transition probability from state $s$ to state $s'$ upon choosing action $i$, $L$ represents the time horizon (i.e., number of actions) of each epoch, $\rho$ is the initial state distribution, and consider running the following algorithm: at the beginning of each epoch, we update a prior distribution over $M$ and takes a sample from the resulting posterior distribution, then we follow the policy which is the optimal for this sampled distribution to take actions sequentially during the epoch. For \emph{any} prior distribution of $M$, we have the expected regret for our algorithm as follows:
    \begin{align}
        \mathbb{E}\left[\operatorname{Regret}\left(T \right)\right]=O(L S \sqrt{A T \log (S A T)}). \nonumber
    \end{align}
\end{lemma}

\begin{proof}[Proof of \thmref{exp-regret}]
For simplicity we consider the optimistic case that we have sampled a sufficient number of times from the decision region $\sP (Y)$, such that all hypotheses with non-zero probability in $\calH$ are enumerated\footnote{In a weaker form, it has been proved in \cite{uaivoi2017} that sampling only the \emph{most likely} hypotheses will lead to just an additive factor to the regret bound. Our framework holds the similar argument, while enjoying a simpler hypothesis generating scheme.}. (Notice that in in Section~\ref{sec:UCI_datasets} and Section~\ref{add-complexity}, we have discussed the impact of running algorithms with different numbers of sampled hypotheses, and show that in practice our framework can still have very competitive performance even with insufficient hypothesis sampling.)

Our problem can then be viewed as a Partially Observable Markov Decision Process (POMDP) with a posterior sampling algorithm, specifically:
\vspace{-0.05in}
\begin{itemize}
    \item Time horizon $L$: The number of feature queries made during each epoch can be considered as the time horizon. This aligns with our definition of $L$ in Theorem~\ref{exp-regret}.
    \item Set of actions $\calA$: Each feature query at a certain time step within an epoch can be considered as an action. Thus the cardinality $A$ in the above bound is equivalent to the number of features $n$.
    \item Set of states $\calB$: Intuitively, we take each action based on current observations from the feature query. Thus, each sequential query set with the resulting outcomes can be considered as a state. The number of possible realizations during an epoch is then equivalent to the cardinality of the state set $S$.
    \item Transition probability $P^{M}_{i}(s' \mid s)$: A belief state $s$ consists of selected queries with observed feature values, such that the state transition probability $P^{M}_{i}(s' \mid s)$ can be fully specified by $\sP [X_i \mid Y]$ as described in Section~\ref{sec:formulation}, without the use of hidden states.
    \item Initial distribution $\rho$: Similarly, this can be fully specified by $\sP [X_i \mid Y]$ and the given $\sP [Y]$.
    \item Reward function $R^{M}$: We consider the reward as the expected utility achieved upon termination: we get zero reward if the algorithm continues to query features (i.e., stopping condition not reached), and get expected reward $\sU (\pi | h) \triangleq \max _{y \in \mathcal{Y}} \mathbb{E}_{y _{\text{true}}}[u(y_{\text{true}}, y) \mid \calS (\pi, h)]$ upon termination by the policy based on the true hypothesis.
    \item Optimal policy for $M$: As illustrated at the beginning of Section~\ref{bound-analysis}, our active planning algorithm EC$^2$ achieves the same utility as the optimal policy under the same environment $\boldtheta$. Thus, $\pi ^{\operatorname{EC}^{2}} _{\boldtheta ^{t}}$ can be considered as the optimal policy for the sampled $M$ in each epoch.\vspace{-0.08in}
\end{itemize}

By replacing the notations on the cardinality of the action space in~\lemref{psrl}, we have the expected regret of Algorithm~\ref{alg:onlinesketch} with the EC$^2$ objective as $\mathbb{E}[\operatorname{Regret}(T)]=O(L S \sqrt{n L T  \log (S n L T )})$, as shown in Theorem~\ref{exp-regret}. Notice that the theorem requires each episode being solved optimally, thus we have adding $L$ into the expression to ensure that the greedy policy achieves the same utility (\textit{i.e.}, full coverage) as the optimal policy. 
\end{proof}

\subsection{Proof of Theorem~\ref{info-regret}}
\label{proof-info-regret}

\begin{proof} We define the \emph{information ratio} of Algorithm~\ref{alg:onlinesketch} as follows:
    \begin{align}
        \Gamma_{t}^{\mathrm{EC}^{2}}
        =\frac{\left(\mathbb{E}\left[\sU \left(\pi^{*}_{\boldtheta ^{\star}}\right) - \sU \left(\pi_{\boldtheta ^{t}}^{\mathrm{EC2}}\right) \right]\right)^{2}}{\mathbb{E}_{h} \left[\mathbb{I}\left(\boldtheta ^{\star}; (\boldtheta ^{t}, \vx _{\pi_{\boldtheta ^{t}}^{\mathrm{EC2}}}, y^{t}, h) \mid \mathbb{O}^{t-1}\right) \right]}, \nonumber
    \end{align}

where $\vx _{\pi_{\boldtheta ^{t}}^{\mathrm{EC2}}}$ represents all the queries and the associated outcomes made by Algorithm~\ref{alg:onlinesketch} with the EC$^{2}$ objective under the sampled environment $\boldtheta ^{t}$, $\mathbb{O} ^{t-1}$ represents all the decision and observation history up to the epoch $t - 1$, and $\sI(\cdot)$ represents the mutual information (i.e., entropy reduction). We omit the $\mathbb{E}_{h}$, $\mathbb{O}^{t-1}$ and $h$ terms in the following to simplify notations.

Simply put, the numerator is the square of the expected \emph{immediate regret} at epoch $t$, and the denominator captures the expected information gain on the true environment $\boldtheta ^{\star}$ by implementing the current policy. The information ratio as a whole can be interpreted as ``the expected regret incurred per bit of information acquired''~\cite{russo2017tutorial}.

Define the maximal information ratio for the algorithm as $\bar{\Gamma}=\max _{t \in\{1, \ldots, T\}} \Gamma_{t}^{\operatorname{EC}^{2}}$. Following the proof in Proposition 1 of \cite{russo2016information}, we derive the bound of Algorithm~\ref{alg:onlinesketch} as follows:
\begin{align}
    \mathbb{E}[\operatorname{Regret}(T)] &=\sum_{t=1}^{T} \mathbb{E}\left[\sU\left(\pi^{*}_{\boldtheta ^{\star}}\right) - \sU\left(\pi_{\boldtheta ^{t}}^{\mathrm{EC2}}\right)\right] \nonumber \\
    &=\sum_{t=1}^{T} \sqrt{\Gamma_{t}^{\operatorname{EC}^{2}} \mathbb{I}\left(\boldtheta ^{\star}; (\boldtheta ^{t}, \mathbf{x}_{\pi_{\boldtheta ^{t}}^{\mathrm{EC2}}}, y^{t}) \right)} \nonumber \\
    & \leq \sqrt{\bar{\Gamma} T \sum_{t=1}^{T} \mathbb{I}\left(\boldtheta ^{\star}; (\boldtheta ^{t}, \mathbf{x}_{\pi_{\boldtheta ^{t}}^{\mathrm{EC2}}}, y^{t}) \right)} \nonumber \\
    &\leq \sqrt{\bar{\Gamma} \mathbb{H}\left(\boldtheta ^{\star} \right) T}. \nonumber
\end{align}

The third step is by Jensen's inequality, and the fourth step is by the chain rule of mutual information. Notice the above bound can be further improved by utilizing the average information ratio or considering the time-varying property of $\Gamma _{t}$ (\cite{devraj2021bit}). A promising next step is to find the closed form of the information ratio (or the ``effective dimension'' of the problem) by applying the auxiliary function of entropy by \cite{chen2017near} against the prediction error rate. In this way we could establish our problem-specific connection between the immediate regret and the information gain, and use it to guide a more efficient sampling.
\thispagestyle{empty}
%
\end{proof}

\section{Details of Extension to Datasets with Continuous Features}
\label{sec:continuous_featurs_details}

In this section, we provide the detail of the method suggested in Section \ref{sec:cont-exp} to extend our framework to datasets with continuous features.
For each random variable (feature) $X_i$ we assume
the $K$ different binary latent random variables $\{Z_{i1}, Z_{i2}, \dots, Z_{iK}\}$, where each of them
corresponds to a threshold for binarizing $X_i$. Given label $Y_j$, we assume the random variable $Z_{ik}$ is distributed by a Bernoulli distribution with parameter $\theta _{ij}^{(k)} = \mathbb{P}\left[X_i, Z_{ik} = 1\mid Y_j \right]$. As before, we may assume some prior information about $\theta _{ij}^{(k)}$ in the form of a prior (Beta) distribution
\footnote{Note that by grouping the binary latent random variables $\{Z_{i,j}\}_{j\in[K]}$ based on feature $X_i$, the $Z_{i,j}$'s are dependent conditioned on a hypothesis $H$. Our analysis in \secref{bound-analysis} no longer applies as $\mathrm{EC}^2$ relies on the conditional independence assumption to achieve the near-optimal cost guarantee for each online session. Nevertheless, one can still apply the proposed algorithm as a heuristic to handle continuous features.}.

In each time $t$, we start by sampling from the posterior distribution of parameters $\theta _{ij}^{(k)}$ for all $i,j,k$. Similar to Algorithm \ref{alg:offlinesketch}, we seek for the feature that maximizes the feature query score. In the case of continuous features, we need to additionally find the best binarization threshold for each feature; this can be done by computing the gains achieved with each threshold and selecting the one that maximizes the gain (or use Algorithm \ref{alg:exp3_threshold_selection} to select thresholds). At the end of the epoch, we update the posterior distributions of all parameters corresponding to the thresholds and the features selected during the epoch.

Note that in each epoch we need to calculate $\mathbb{P}\left[h\right]$ for all available hypotheses $\{h\}$, which typically requires access to the parameters $\theta _{ij}$:

\begin{align}
\mathbb{P}\left[h\right] = \sum_{Y_j \in \mathcal{Y}} \mathbb{P}\left[h|Y_j \right]\mathbb{P}\left[Y_j\right], \nonumber
\end{align}
where
\begin{align}
\mathbb{P}\left[h|Y_j\right] = \prod_{i \in \mathcal{Q}} \mathbb{P}\left[X_i=h_i|Y_j\right] = \prod_{i \in \mathcal{Q}} \theta _{ij}^{h_i}(1-\theta _{ij})^{(1-h_i)}. \nonumber
\end{align}

In the continuous setting, for each pair $(X_i, Y_j)$, we have $K$ different parameters $\theta _{ij}^{(k)}$. In practice, we may use the weighted average value of these parameters as an estimation of $\theta _{ij}$, according to the number of times  the corresponding thresholds are used for the label $Y_j$.

\section{Faster UFODT}
\label{sec:feature_selection}
As mentioned before, in each epoch of our online decision tree learning framework, we aim to optimize the utility of features, i.e., to maximize $\sU (\vx _{\calF})$ with the cheapest query set $\calF$. We do this optimization by greedily maximizing the score of features based on an information acquisition (surrogate) function (line 4 of Algorithm \ref{alg:offlinesketch}). The computational cost of UFODT in each time step is dominated by the computation of such scores which is determined by the total number of hypotheses. For instance, in case of UFODT-EC\textsuperscript{2}, calculating $\Delta_{E C^{2}}\left(u \mid \mathbf{x}_{\mathcal{F}}\right)$ takes $\mathcal{O}(|\mathcal{H}|^2)$ time for a binary feature $u$ where $|\mathcal{H}|$ is the number of hypotheses \cite{nipsnoisy2013}. $|\mathcal{H}|$ grows exponentially with the number of features. As a result, we develop two  solutions to make UFODT faster: i) we  reduce the number of score calculations, and ii) we reduce the number of features. In what follows we present a practical method for them. 
\paragraph{Feature selection.} Feature selection is widely used in batch machine learning to improve the efficiency of learning algorithms and also to prevent overfitting. However, the conventional feature selection methods are not well-suited for online learning scenarios. To our knowledge, there has not been much work on feature selection for streaming data points. The work in \cite{wang2013online} proposes an Online Feature Selection (OFS) algorithm that is able perform feature selection from partial inputs. Their algorithm uses  $\epsilon$-greedy to select a constant number of features in each time step. Specifically, they train an online perceptron classifier, and in each time step, features with highest weights are chosen with probability $1-\epsilon$. Otherwise, a random set of features is selected (w.p. $\epsilon$) to allow for exploration. To train the weights with partial inputs, they use an unbiased estimate of each feature. This algorithm is not directly applicable to our framework as we query different number of features in each time step. So, we modify it and use this new modified OFS as a component within UFODT. At each time step $t$, we start by selecting a subset $\mathcal{C}_t \subset \gQ$ 
of features according to $\epsilon$-greedy based on the current weights. We then use UFODT as before, except that we only query from the features in $\mathcal{C}_t$ in the planning phase. In other words, we have two stages of feature selection: first the OFS algorithm selects the features available for querying, and then UFODT queries a subset of those selected features according to the information acquisition function. At the end of time $t$, we need to update the weights of our OFS algorithm. For that, we use the following estimate of each feature $x_i$ of data point $\vx ^t$:

\[
\hat{x}_i = \frac{\mathbbm{1}\{(i \in \mathcal{F}) \wedge (i \in \mathcal{C}_t)\}x_i}{\frac{B}{n}\epsilon + \mathbbm{1}\{(i \in \mathcal{F}) \wedge (i \in \mathcal{C}_t) \} (1-\epsilon)},
\]

\noindent where $B$ is the number of features selected by OFS, and $\mathcal{F}$ is the feature set queried by UFODT. 

\section{Additional Experimental Results}
\label{sec:additional_experiments}
\subsection{Feature Selection}
In this section, we study
application of our feature selection scheme (Appendix \ref{sec:feature_selection}) to the UFODT framework. We use the same datasets as in Section \ref{sec:cont-exp}. We compare the feature querying cost and test utility achieved when using our OFS method (shown by UFODT-\emph{criterion}-OFS) with those achieved by VFDT, EFDT, and UFODT (without feature selection and using exhaustive search over thresholds). In Figures \ref{fig:cost_in_progress_diabetes_OFS}, \ref{fig:cost_in_progress_wdbc_OFS}, and \ref{fig:cost_in_progress_fetal_OFS}, we observe that using OFS clearly reduces the querying cost (and thereby running time) of UFODT for both EC\textsuperscript{2} and IG. Using feature selection causes decrease of test utility for UFODT-EC\textsuperscript{2}-OFS (especially for the Diabetes dataset shown in Figure \ref{fig:test_util_in_progress_diabetes_OFS}). However, for Breast Cancer and Fetal Health (Figures \ref{fig:test_util_in_progress_wdbc_OFS} and \ref{fig:test_util_in_progress_fetal_OFS}), we observe that UFODT-EC\textsuperscript{2}-OFS has very close test performance to that of UFODT-EC\textsuperscript{2} or reaches the test performance of UFODT-EC\textsuperscript{2} at later training time steps. The test utility of UFODT-IG-OFS is very close to that of UFODT-IG and even better in some time steps. These results indicate that we can use our feature selection scheme together with UFODT to reduce the computational cost of our framework for datasets with large number of features.

\begin{figure}[htb!]
    \hspace{-5mm}
    \subfigure[D - Cost]
    {
        \includegraphics[width=0.5\columnwidth]{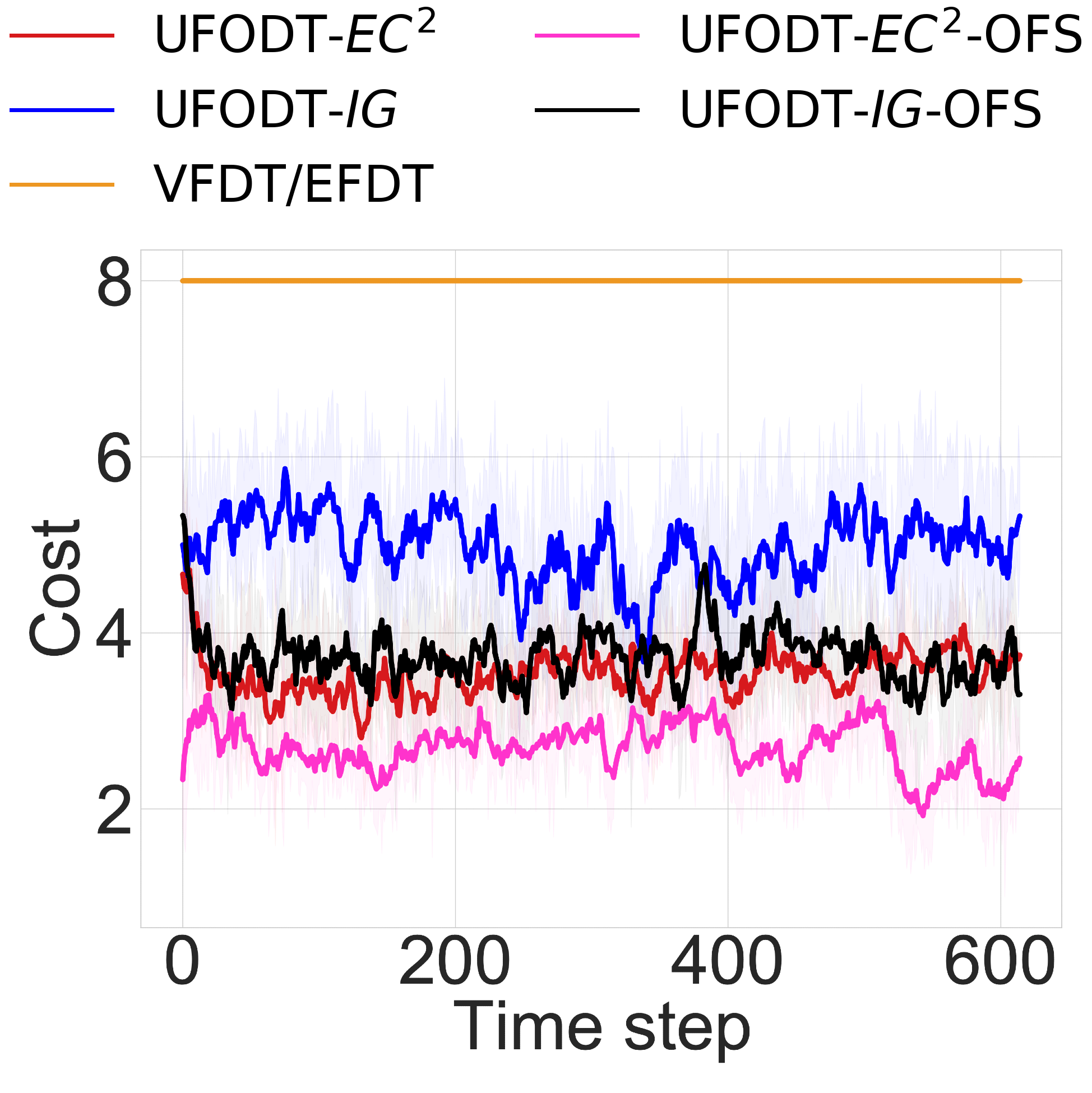}
        \label{fig:cost_in_progress_diabetes_OFS}
    }
    \hspace{-6mm}
    \subfigure[D - Test utility]
    {
        \includegraphics[width=0.5\columnwidth]{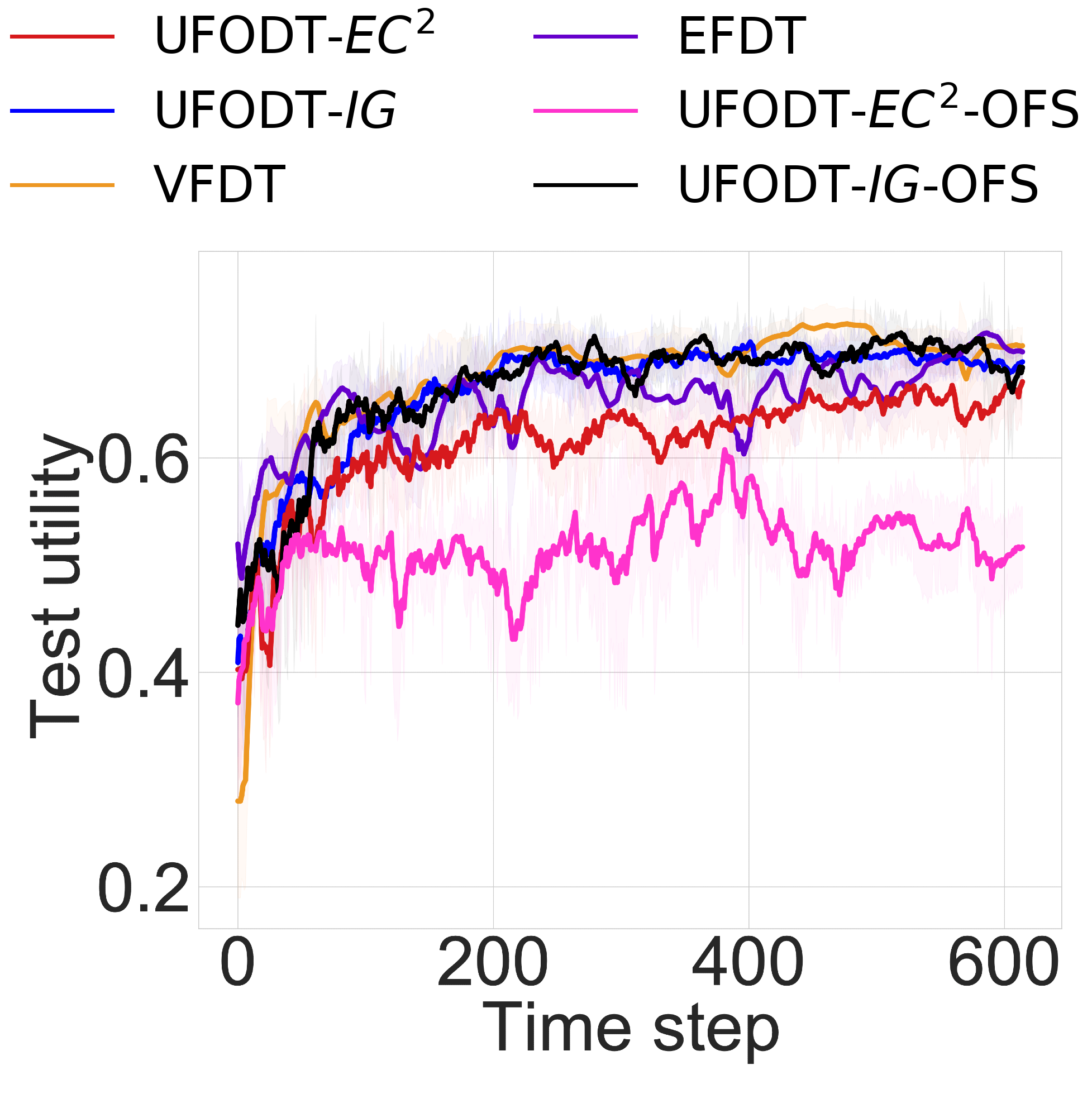}
        \label{fig:test_util_in_progress_diabetes_OFS}
    }
    \subfigure[B - Cost]
    {
        \hspace{-4.5mm}
        \includegraphics[width=0.5\columnwidth]{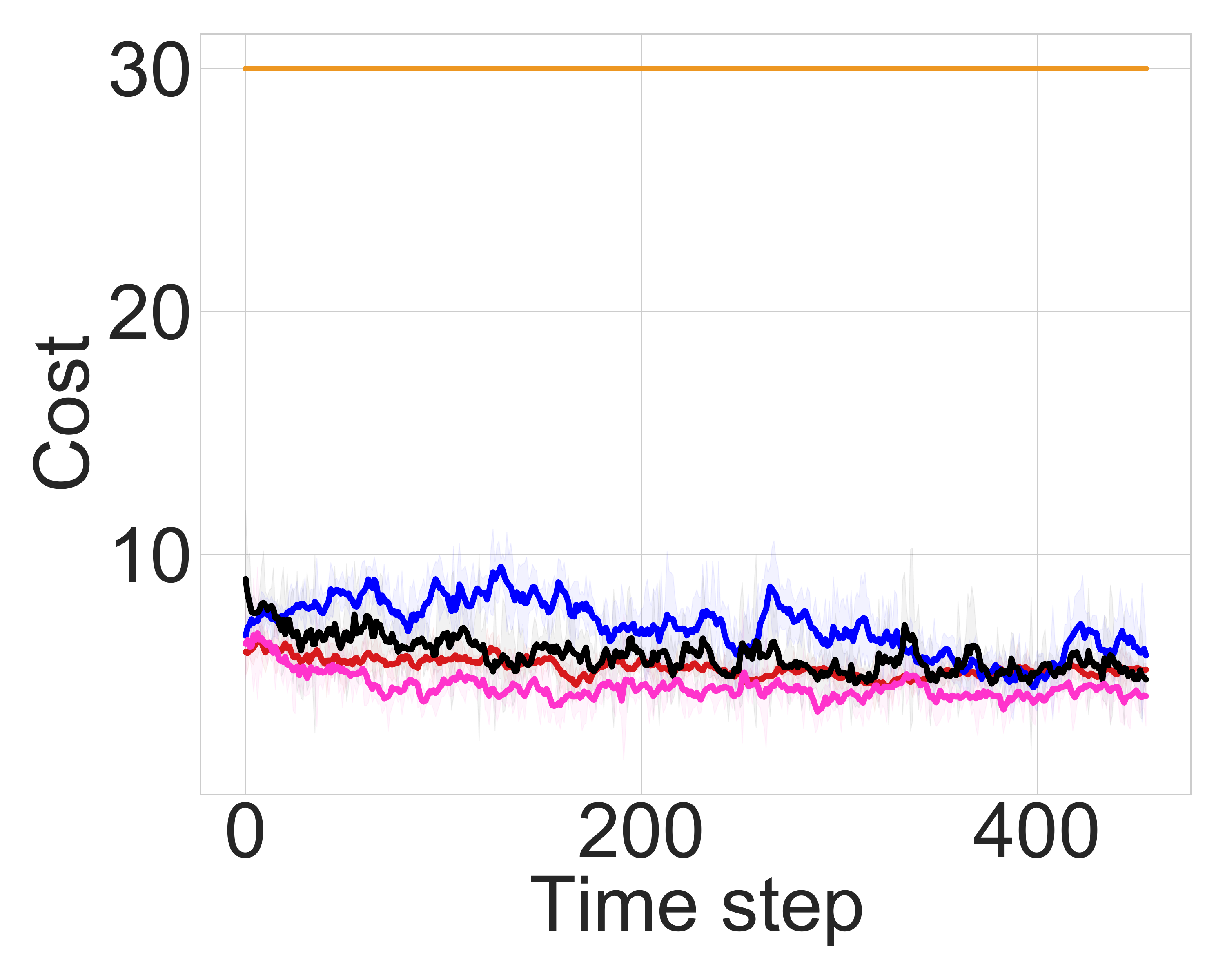}
        \label{fig:cost_in_progress_wdbc_OFS}
    }
    \hspace{-6mm}
    \subfigure[B - Test utility]
    {
        \includegraphics[width=0.5\columnwidth]{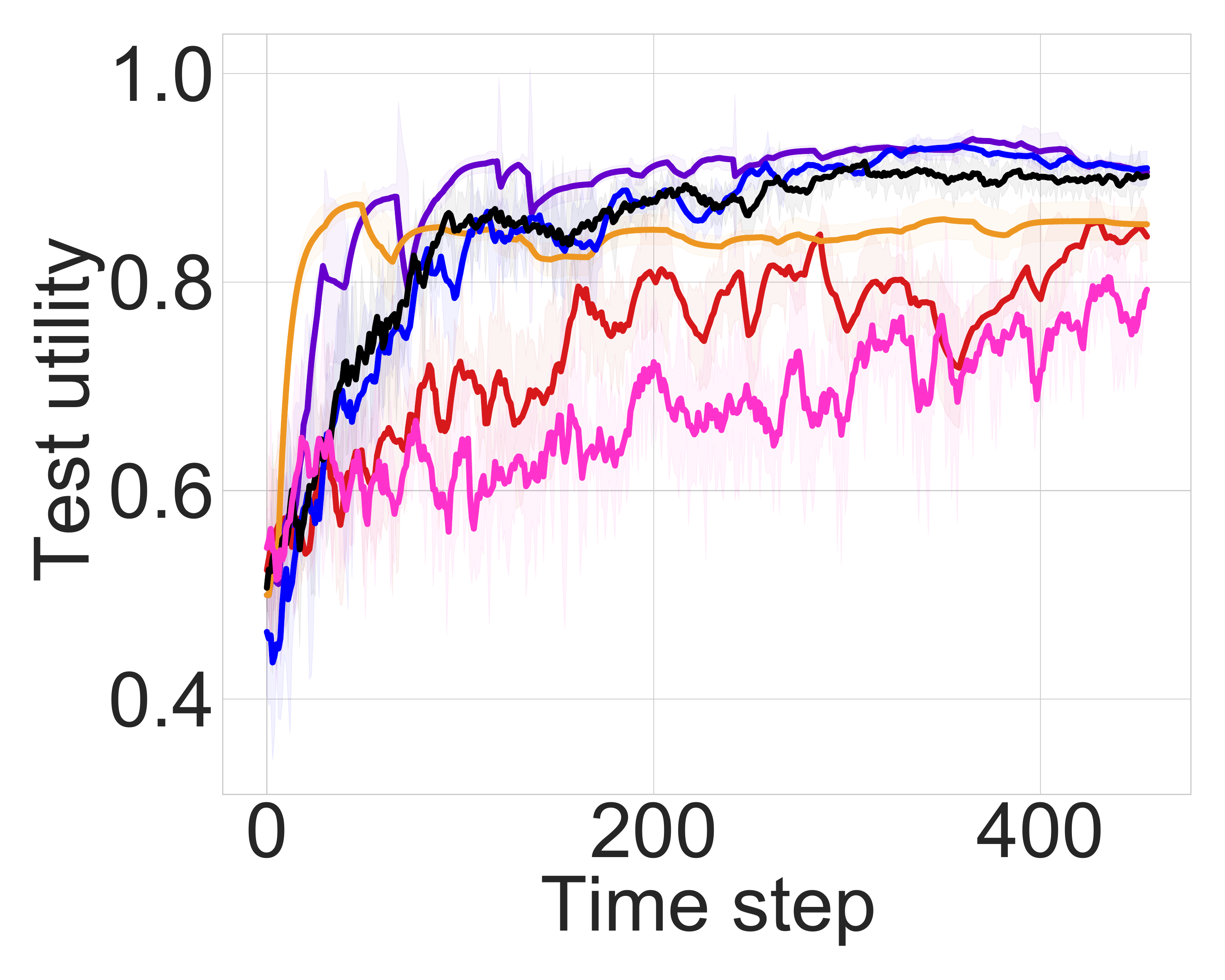}
        \label{fig:test_util_in_progress_wdbc_OFS}
    }
    \subfigure[F - Cost]
    {
        \hspace{-4.5mm}
        \includegraphics[width=0.5\columnwidth]{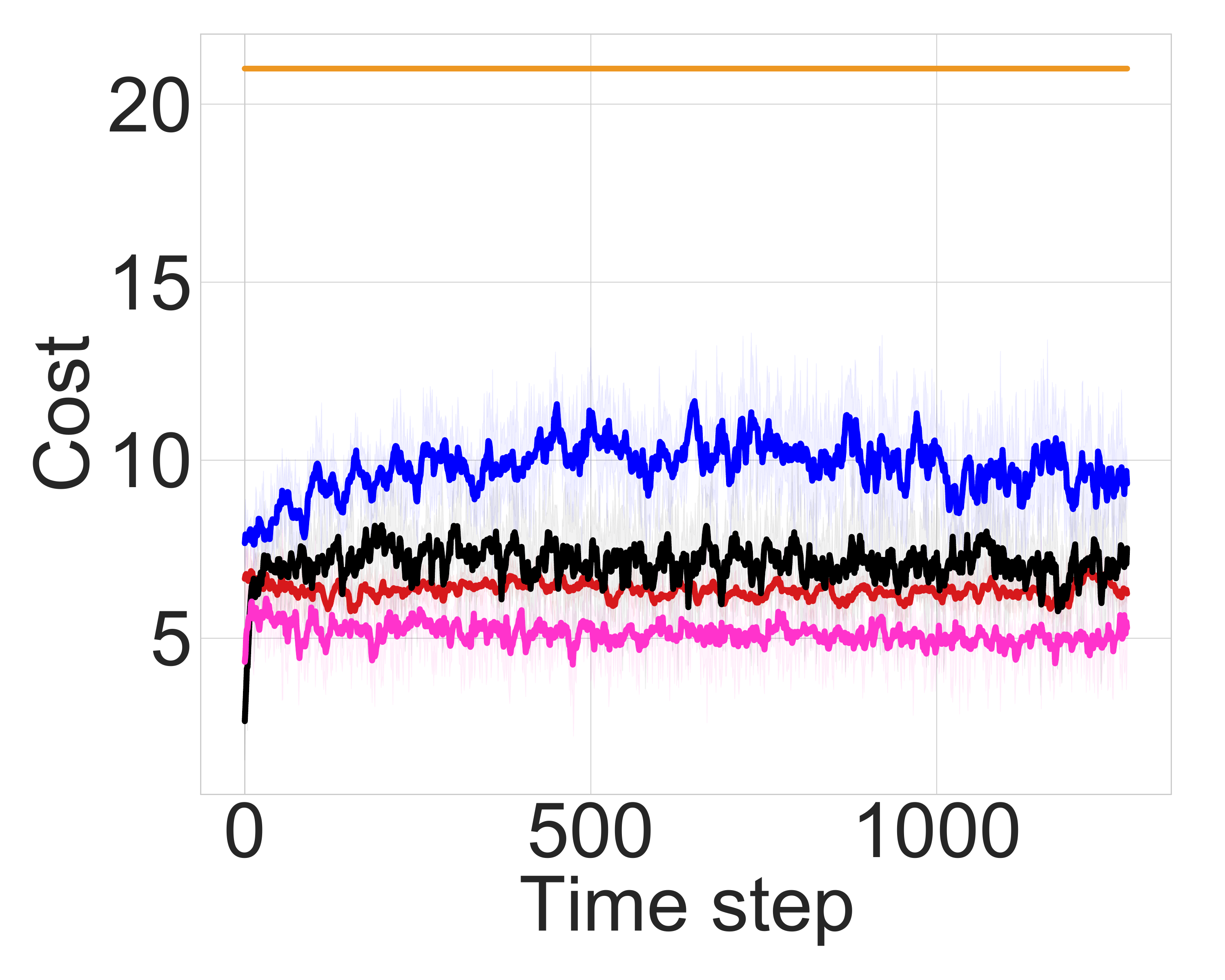}
        \label{fig:cost_in_progress_fetal_OFS}
    }
    \subfigure[F - Test utility]
    {
        \hspace{-2.5mm}
        \includegraphics[width=0.5\columnwidth]{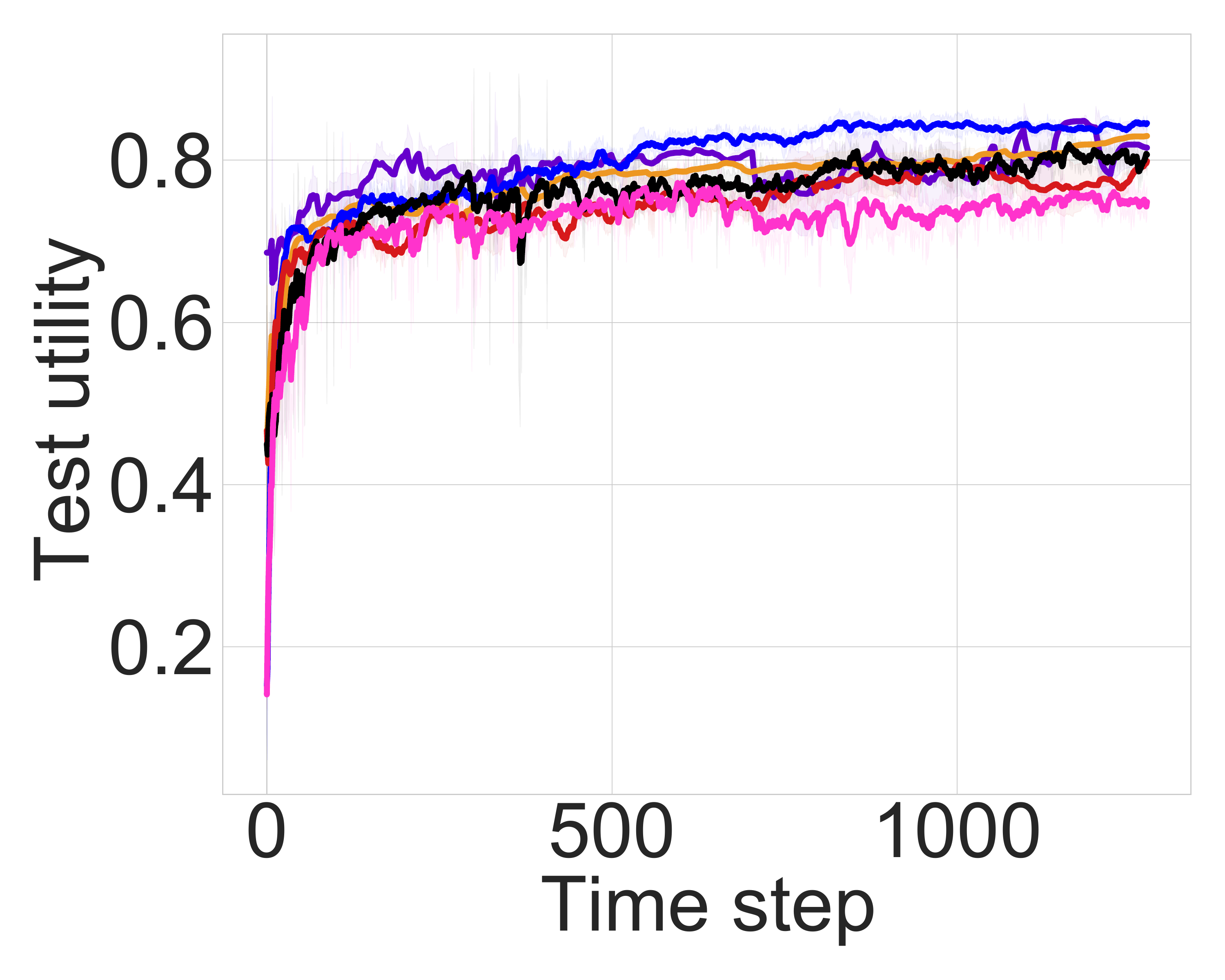}
        \label{fig:test_util_in_progress_fetal_OFS}
    }
    \caption{The cost (a,c,e) and test utility (b,d,f) during the training process for Prima Indians Diabetes (D), Breast Cancer (B) and Fetal Health (F) datasets when using our UFODT framework together with feature selection. Our feature selection scheme generally maintains competitive test utilities while having lower feature query costs and lower time complexity.}
    \label{fig:cost_in_progress_continuous_OFS}
\end{figure}

\subsection{Utility and Cost for Different Numbers of Sampled Hypotheses.} Figure \ref{fig:UCI_Cost_vs_Hypo} shows the average cost (i.e., the number of features queried within an epoch) during training as a function of the total number of hypotheses sampled for LED, Zoo and Heart datasets. Figure \ref{fig:UCI_Utility_vs_Hypo} shows the total utility during training versus the total number of sampled hypotheses. We observe that UFODT-EC\textsuperscript{2} yields the lowest cost in the three cases and its utility is competitive compared to the best results. This observation shows that UFODT-EC\textsuperscript{2} tends to find more informative features to query, meaning that with less number of features (lower cost) it can reach a high utility. UFODT-IG also yields a better cost than random feature selection and UFODT-US. On the other hand, both UFODT-random and UFODT-US require a large number of queries which does not necessarily help them to attain high utilities. 

\subsection{Train Utility During Training.} Figure \ref{fig:UCI_train_utility_in_progress} illustrates the utility (accuracy or F-measure) on the training datasets during learning. 
For all the three datasets, we observe that
 UFODT-EC\textsuperscript{2} reaches a very competitive utility during training with a much lower cost. 
The number of sampled hypotheses are similar to that of Figure \ref{fig:UCI_test_utility_in_progress}.

\subsection{The Impact of the Number of Sampled Hypotheses on Concept Drift Experiments}\label{add-complexity}
Since UFODT relies on hypothesis sampling~(see Algorithm~\ref{alg:sampling}), 
we further investigate how its performance is affected by the number of sampled hypothesis. The results are presented in Figure~\ref{fig:stagger-fig3} (left: non-stationary posterior sampling, right: standard posterior sampling), using the Stagger dataset, in complement to the concept drift experiments in Section~\ref{sec:exp-drift}. As expected, by increasing  the number of sampled hypothesis, the test utility also increases. However, the test utility usually saturates at some  early stage (e.g., when the number of sampled hypothesis is around 9 in this case). This implies that enumerating all the possible hypothesis may not be necessary, so that sampling can help to reduce the running time to a great extent.

\subsection{The Impact of Priors on Concept Drift Experiments} 
\label{sec:priorvsutility}
UFODT can easily incorporate expert's knowledge by using the informative priors, enjoying superior flexibility over classic decision tree algorithms. To simulate different experts, we generate a collection of priors that interpolate between the uniform prior~(uninformative) and the ``optimal'' prior~(expert). We report the average test utility for different priors in Figure~\ref{fig:priorvsutility}. 
We observe that as the quality of the prior improves, the average test utility increases and surpasses EFDT by a larger margin. In general, with more informative priors, our methods perform better, which is consistent with the prior dependent regret bound in Theorem~\ref{info-regret}.
\begin{figure*}[t]
    \centering
    \vspace{-0.2cm}
    \includegraphics[width=1.0\textwidth]{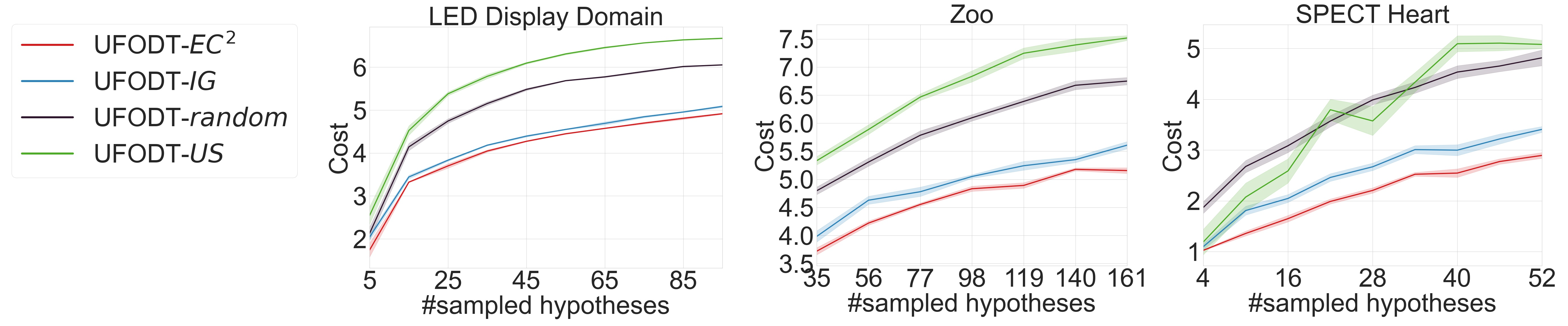}
    \vspace{-0.35cm}
    \caption{Training cost vs. \#sampled hypotheses. UFODT-EC\textsuperscript{2} yields the lowest cost in all three cases.}
    \label{fig:UCI_Cost_vs_Hypo}
    \vspace{-0.2cm}
\end{figure*}
\begin{figure*}[t]
    \centering
    \includegraphics[width=1.0\textwidth]{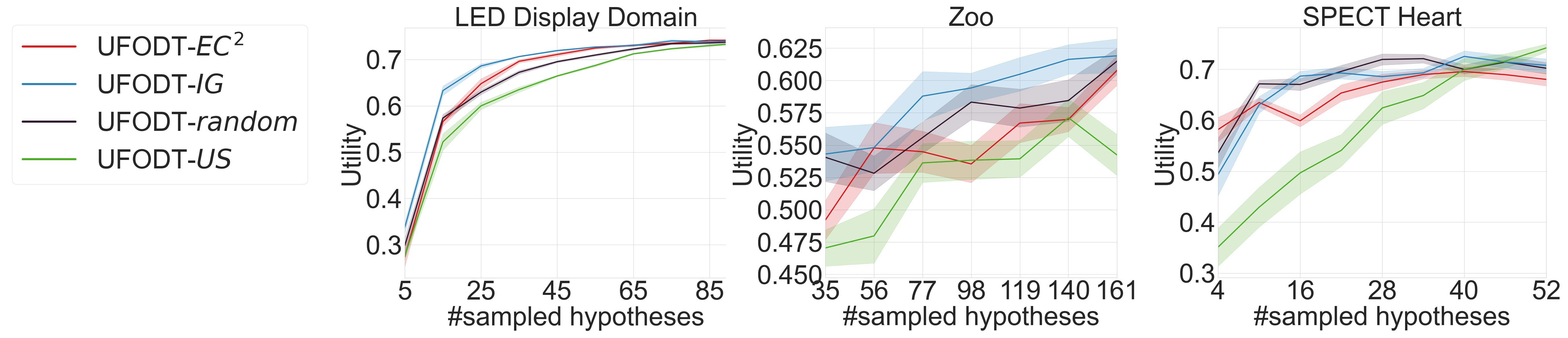}
    \vspace{-0.35cm}
    \caption{Training utility vs. \#sampled hypotheses.
    The utility achieved by UFODT-EC\textsuperscript{2} is similar to or even better than the other methods, while having a lower cost.
    }
    \label{fig:UCI_Utility_vs_Hypo}
    \vspace{-0.3cm}
\end{figure*}
\begin{figure*}[t]
    \centering
    \includegraphics[width=1.0\textwidth]{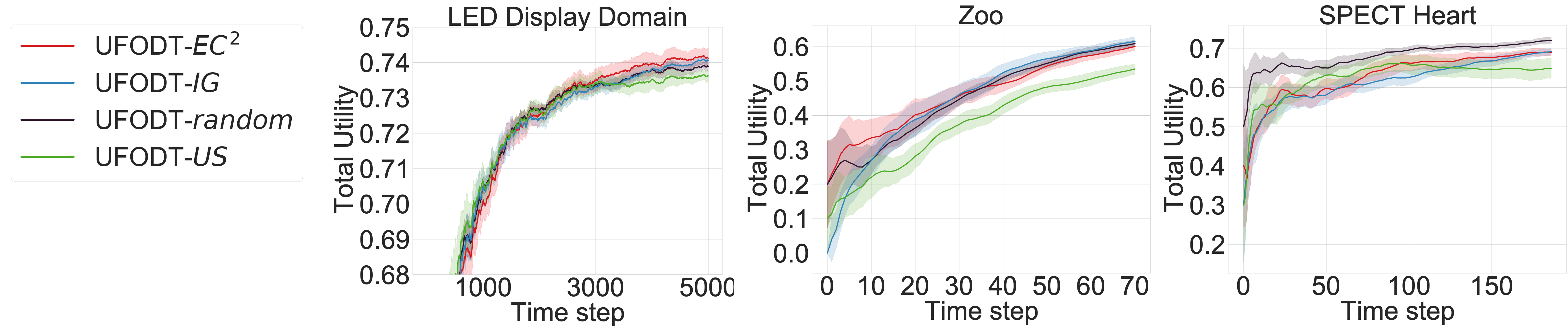}
    \vspace{-0.35cm}
    \caption{Training utility during training:
    UFODT-EC\textsuperscript{2} reaches a good utility during training steps with low cost. 
    }
    \label{fig:UCI_train_utility_in_progress}
\end{figure*}

\begin{figure*}[t]
    \centering
    \subfigure[]{
     \includegraphics[width=0.45\textwidth]{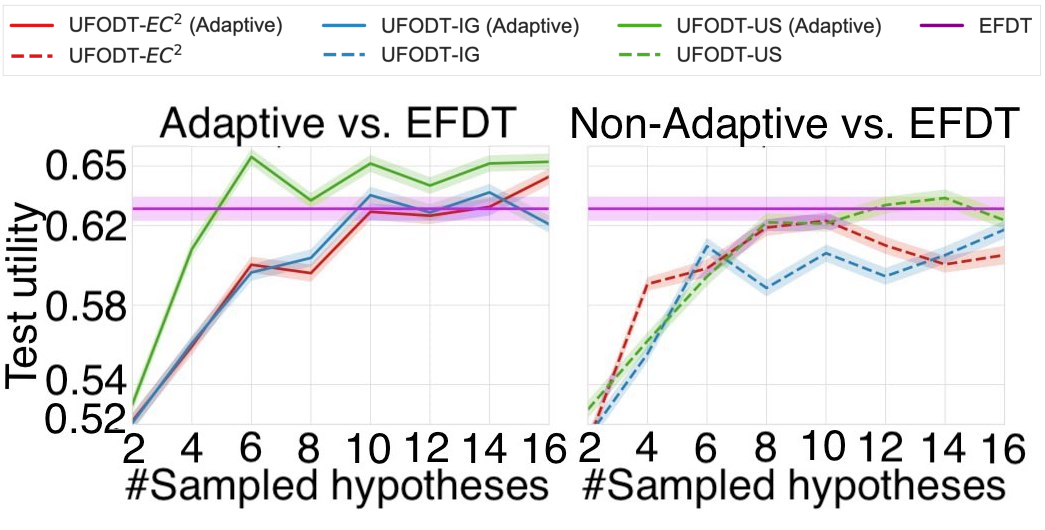}
    \label{fig:stagger-fig3}
     }
    \subfigure[]{
     \includegraphics[width=0.5\textwidth]{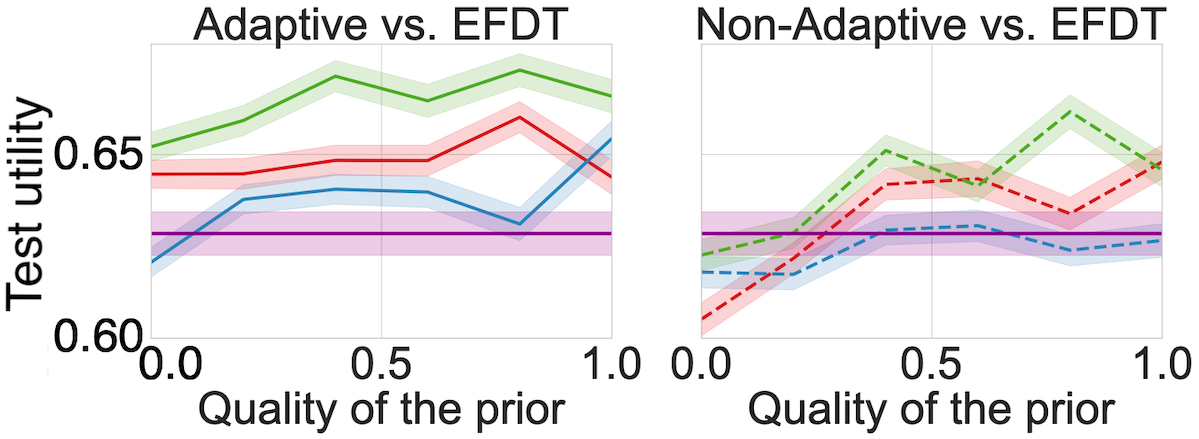}
     \label{fig:priorvsutility}
     }
    \caption{\ref{fig:stagger-fig3}: The effect of the number of sampled hypotheses on the test utility using the Stagger dataset. \ref{fig:priorvsutility}: Quality of prior vs. test utility using the Stagger dataset. Along $x$-axis, larger value corresponds to more accurate prior.}
\end{figure*}

\end{document}